\documentclass[11pt,reqno]{amsart}
\usepackage[margin=1in]{geometry}
\usepackage{amsmath,amssymb,amsthm,mathtools}
\usepackage{booktabs}
\usepackage{microtype}
\usepackage[colorlinks=true,linkcolor=blue,citecolor=blue,urlcolor=blue]{hyperref}

\theoremstyle{plain}
\newtheorem{theorem}{Theorem}[section]
\newtheorem{proposition}[theorem]{Proposition}
\newtheorem{corollary}[theorem]{Corollary}
\newtheorem{lemma}[theorem]{Lemma}
\newtheorem{conjecture}[theorem]{Conjecture}
\theoremstyle{definition}
\newtheorem{definition}[theorem]{Definition}
\theoremstyle{remark}
\newtheorem{remark}[theorem]{Remark}

\newcommand{\R}{\mathbb{R}}
\newcommand{\Sd}{\mathbb{S}^{d-1}}
\newcommand{\PP}{\mathbb{P}}
\newcommand{\EE}{\mathbb{E}}
\newcommand{\eps}{\varepsilon}
\newcommand{\Lip}{\operatorname{Lip}}
\newcommand{\relu}{\operatorname{ReLU}}
\newcommand{\NN}{\mathcal{N}}
\newcommand{\GG}{\mathcal{G}}
\newcommand{\AAA}{\mathcal{A}}
\newcommand{\FF}{\mathcal{F}}
\newcommand{\inn}[2]{\langle #1,#2\rangle}
\DeclareMathOperator{\clip}{clip}
\DeclareMathOperator{\Var}{Var}
\DeclareMathOperator{\aff}{aff}

\newcommand{\proofbutton}[1]{\hyperref[proof:#1]{\textnormal{\footnotesize\textsf{[proof $\to$]}}}}
\newcommand{\stmtbutton}[1]{\hyperref[#1]{\textnormal{\footnotesize\textsf{[$\leftarrow$ statement]}}}}

\title[Law of robustness for two-layer networks]{A law of robustness for two-layer neural networks with arbitrary weights}
\author{Yitzchak Shmalo}
\address{Einstein Institute of Mathematics, The Hebrew University of Jerusalem, Givat Ram, Jerusalem, Israel}
\email{yitzchak.shmalo@gmail.com}
\date{July 7, 2026}
\hypersetup{
  pdftitle={A law of robustness for two-layer neural networks with arbitrary weights},
  pdfauthor={Yitzchak Shmalo},
  pdfkeywords={law of robustness, two-layer neural networks, ReLU, Lipschitz interpolation, arbitrary weights}
}

\begin{document}

\begin{abstract}
Bubeck, Li and Nagaraj conjectured that, for generic data, any two-layer neural network with $m$ neurons that fits $n$ noisy labels must have Lipschitz constant at least of order $\sqrt{n/m}$, with no restriction on the size of the weights. Bubeck and Sellke proved a universal version of this law for Lipschitz-parameterized classes, but under a polynomial bound on the parameters; at depth three that boundedness hypothesis is genuinely necessary. The two-layer unbounded-weight case therefore requires a different argument.

We prove the conjectured law, up to one logarithmic factor, for every continuous piecewise-linear activation, in particular for ReLU networks. For data drawn either uniformly from $\mathbb S^{d-1}$, $d\ge3$, or from $N(0,I_d/d)$, labels in $[-1,1]$ with conditional noise level $\sigma^2>0$, and any fixed width-$m$ two-layer network with arbitrary real weights, biases and affine skip connection, fitting the data $\varepsilon$ below the noise floor forces
\[
        \operatorname{Lip}(f)\ge c\,\varepsilon\sqrt{\frac{n}{\bar m\log(C\bar mnd/\varepsilon)}},
        \qquad \bar m=(K-1)m+1,
\]
with high probability. We also prove a finite-horizon simultaneous-width version and a realized-kink-count version: on one high-probability event, every realized two-layer piecewise-linear function with $k(f)\le n$ distinct kink hyperplanes obeys the same bound with $\bar m$ replaced by $k(f)+1$, irrespective of how many redundant hidden units were used to parameterize it.

The proof replaces parameter-space covering, which is impossible for unbounded weights, by a function-space covering. The central deterministic ingredient is a rigidity lemma: on $B_2$, and on $\mathbb S^{d-1}$ for $d\ge3$, the coefficient of each canonical kink is controlled by the Lipschitz constant of the realized function, because kinks supported on distinct hyperplanes cannot cancel at generic points. This yields a bounded canonical representation and hence the required entropy bound. We also show why the sphere argument genuinely excludes $d=2$, give a two-layer ReLU interpolant with $O(1)$ Lipschitz constant at width $2n$ in the high-dimensional separated regime, and state the precise concentration/localization hypotheses under which the Gaussian proof extends beyond the Gaussian measure.
\end{abstract}

\maketitle

\section{Introduction}\label{sec:intro}

A function that fits $n$ noisy labels and is to be robust --- small Lipschitz constant --- needs capacity. Bubeck, Li and Nagaraj \cite{BLN} made this precise for the basic architecture of the subject. Let
\begin{equation}\label{eq:class}
\NN_m \;=\;\Bigl\{\,f(x)=\sum_{k=1}^{m}a_k\,\psi(\inn{w_k}{x}+b_k)+\inn{v}{x}+c\;:\;a_k,b_k,c\in\R,\ w_k,v\in\R^d\,\Bigr\}
\end{equation}
be the class of two-layer networks of width $m$ with activation $\psi$, with \emph{no restriction whatsoever} on the magnitudes of the weights. The affine part $\inn{v}{x}+c$ only enlarges the class studied in \cite{BLN}; all results below hold a fortiori without it.

\begin{conjecture}[Bubeck--Li--Nagaraj \cite{BLN}, Conjecture 1]\label{conj:bln}
Let $\psi$ be any Lipschitz activation. For $x_1,\dots,x_n$ independent uniform on $\Sd$ (or $N(0,I_d/d)$) and $y_1,\dots,y_n$ independent uniform on $\{-1,+1\}$, with high probability, any $f\in\NN_m$ fitting the data must satisfy
\[
\Lip_{\Sd}(f)\;\ge\;c\,\sqrt{n/m}.
\]
\end{conjecture}
\hyperref[thm:main]{[main results $\to$]}

The interpretation is that robust interpolation should require on the order of one neuron per data point, while non-robust interpolation can require far fewer neurons in high dimension. Bubeck and Sellke \cite{BS} proved a far-reaching generalization: for any function class admitting a Lipschitz parameterization by $p$ real parameters of polynomial size, and for covariate distributions satisfying isoperimetry, fitting below the noise floor forces $\Lip(f)\gtrsim\varepsilon\sqrt{nd/p}$ up to logarithmic factors. For width-$m$ two-layer networks, $p=\Theta(md)$, giving the desired $\sqrt{n/m}$ scaling under the polynomial-weight hypothesis. That hypothesis is not merely technical at larger depth: Bubeck and Sellke construct three-layer unbounded-weight networks that violate the law. Wu, Huang and Zhang \cite{WHZ} subsequently extended robustness laws beyond isoperimetric data, under polynomially bounded parameters. The unbounded-weight two-layer ReLU case posed by Conjecture~\ref{conj:bln} remains the natural boundary case.

This paper proves the law for every continuous piecewise-linear activation, in particular for ReLU networks, up to one logarithmic factor. The price of the logarithm is explicit throughout; we do not claim the log-free lower bound. Section~\ref{sec:sharplog}, in the sphere model, sharpens the logarithm itself: in the regime $n\gtrsim md^2\log(md)$ the factor $\log(Cmnd)$ improves to $\log(Cmd)$ --- the sample size leaves the logarithm --- and no single-scale packing argument can show that any logarithm is necessary. Section~\ref{sec:genact} proves projection-capacity floors valid for \emph{every} Lipschitz activation: the conjecture holds at width one with margin $\sqrt{n/\log(nd)}$, at width two on the whole admissible dimension range, and at width three for $d\ge C\log(nd)$, log-free. Beyond $d\sim n/(m\log(nd))$ the projection method is exhausted --- its net cost reaches the label budget --- and for general activations at width $m\ge2$ that regime remains open; separately, a linear-activation interpolant shows that no floor exceeding $C\sqrt n$ can hold once $d\gtrsim n$. Section~\ref{sec:band} states the one open multiplier estimate (Conjecture~\ref{conj:M}) to which the log-free conjecture reduces in the critical band of widths; the reduction itself, together with the unconditional structure surrounding it --- occupancy, serving capacity, pile-up rigidity, cap mass, forced depth, an affine supremum identity, and the single-direction case settled for \emph{every} Lipschitz activation, with no logarithm --- is developed in the supplementary note \cite{supp}. The reduction is representation-free, so the remaining obstacle for arbitrary activations coincides with the ReLU one.

\subsection{Standing notation and conventions}\label{ssec:notation}
We fix these throughout.
\begin{itemize}
\item $\inn{\cdot}{\cdot}$ and $\|\cdot\|$ are the Euclidean inner product and norm on $\R^d$; $\mathbb S^{d-1}=\{x:\|x\|=1\}$; $B_R=\{x:\|x\|\le R\}$.
\item $\Lip_D(f)=\sup_{x\ne x'\in D}|f(x)-f(x')|/\|x-x'\|$ is the Euclidean Lipschitz constant of $f$ on a set $D$.
\item $\relu(t)=\max(t,0)$. A function $\psi:\R\to\R$ is \emph{continuous piecewise linear with $K$ pieces} if it is continuous and there are breakpoints $\tau_1<\dots<\tau_{K-1}$ such that $\psi$ is affine on each of $(-\infty,\tau_1),(\tau_1,\tau_2),\dots,(\tau_{K-1},\infty)$; $\relu$ is the case $K=2$, $\tau_1=0$. Continuity is part of the definition and is used in Lemma~\ref{lem:pl}.
\item For a unit vector $u$ and $t\in\R$, $H_{u,t}=\{x:\inn{u}{x}=t\}$ is the hyperplane with unit normal $u$ at signed distance $t$ from the origin.
\item $\clip(t)=\max(-1,\min(1,t))$ is the projection of $\R$ onto $[-1,1]$; it is $1$-Lipschitz.
\item $\psi$ always denotes the network activation; $\psi_2$ (with a subscript) denotes the sub-Gaussian Orlicz norm $\|Z\|_{\psi_2}=\inf\{s>0:\EE e^{Z^2/s^2}\le2\}$.
\item $c,C,c_0,c_2,C_0,C',\kappa$ denote positive absolute constants; $c,C$ may change from line to line, while the subscripted constants are fixed once chosen.
\item Every supremum over a class of networks or of Lipschitz functions that appears below inside an expectation is over a class that is separable in the uniform norm --- the weights range over a finite-dimensional space with continuous dependence, or over a ball of Lipschitz functions --- so it equals the supremum over a fixed countable dense subset and is measurable; we write $\sup$ without further comment.
\end{itemize}
We write $\psi$ for a piecewise-linear activation with $K$ pieces and put $\bar m:=(K-1)m+1$; for ReLU $\bar m=m+1$. We use two data models, both from \cite{BLN}:
\begin{itemize}
\item[(S)] $\mu=$ uniform probability measure on $\Sd$, with $d\ge3$, and domain $D=\Sd$;
\item[(G)] $\mu=N(0,I_d/d)$ and domain $D=B_2$.
\end{itemize}
Data are $(x_i,y_i)_{i\le n}$ i.i.d.\ with $x_i\sim\mu$, $y_i\in[-1,1]$, and noise level $\sigma^2:=\EE\,\Var(y\mid x)>0$. Independent $\pm1$ labels are the special case $\sigma^2=1$.

\subsection{The result}

\begin{theorem}[Fixed-width law]\label{thm:main}
There are absolute constants $C_0,c_0>0$ such that the following holds. Let $\psi$ be continuous piecewise linear with $K$ pieces, put $\bar m=(K-1)m+1$, let $\eps\in(0,\sigma^2]$ and $\delta\in(0,1)$, and assume
\[
        n\ge C_0\eps^{-2}\log(8/\delta).
\]
In model (G) assume additionally $d\ge2\log(8n/\delta)$. Then, with probability at least $1-\delta$, every $f\in\NN_m$ with arbitrary weights and
\[
\frac1n\sum_{i=1}^n\bigl(f(x_i)-y_i\bigr)^2\le\sigma^2-\eps
\]
satisfies
\[
\Lip_D(f)\ge c_0\eps\sqrt{\frac{n}{\bar m\log\!\bigl(C_0\bar mnd/\eps\bigr)}}.
\]
\end{theorem}
\proofbutton{main}

\begin{corollary}[BLN for piecewise-linear activations, up to one logarithm]\label{cor:bln}
In the setting of Conjecture~\ref{conj:bln}, with $\psi$ continuous piecewise linear and $d\ge3$ in the sphere model, there are constants $c_1,C_1>0$ depending on $\psi$ only through its number of pieces such that, with probability at least $1-\delta$, for $n\ge C_1\log(8/\delta)$ every $f\in\NN_m$ with arbitrary weights that fits the data exactly, or merely has empirical mean squared error at most $1/2$, satisfies
\[
\Lip_{\Sd}(f)\ge c_1\sqrt{\frac{n}{m\log(C_1mnd)}}.
\]
\end{corollary}
\proofbutton{bln}

\begin{corollary}[Finite-horizon simultaneous widths]\label{cor:simul}
Fix an integer $M\ge1$. Under the hypotheses of Theorem~\ref{thm:main}, with
\[
        n\ge C_0\eps^{-2}\log(8M/\delta)
\]
and, in model (G), $d\ge2\log(8nM/\delta)$, there is an event of probability at least $1-\delta$ on which the conclusion of Theorem~\ref{thm:main} holds simultaneously for every width $1\le m\le M$ and every $f\in\NN_m$.
\end{corollary}
\proofbutton{simul}

The width enters the proof only through the number of distinct hyperplanes on which the realized function has a kink. For a two-layer piecewise-linear function $f$, let $k(f)$ be the number of distinct hyperplanes $H_{u,t}$ carrying a nonzero kink of the canonical representation inside $D$. After Lemma~\ref{lem:pl}, a width-$m$ network has $k(f)\le(K-1)m$, but redundant or cancelling parameterizations may have $k(f)$ much smaller.

\begin{theorem}[Realized kink-count law]\label{thm:kink}
There are absolute constants $C_0,c_0>0$ such that, in the setting of Theorem~\ref{thm:main}, if
\[
        n\ge C_0\eps^{-2}\log(8n/\delta)
\]
and, in model (G), $d\ge2\log(8n/\delta)$, then with probability at least $1-\delta$ every two-layer piecewise-linear network $f$, of arbitrary width and arbitrary weights, with $k(f)\le n$ and fitting the data $\eps$ below the noise floor, satisfies
\[
\Lip_D(f)\ge c_0\eps\sqrt{\frac{n}{(k(f)+1)\log\!\bigl(C_0(k(f)+1)nd/\eps\bigr)}}.
\]
\end{theorem}
\proofbutton{kink}

The kink-count theorem is stronger than the fixed-width theorem for realized functions with $k(f)\le n$; Theorem~\ref{thm:main} is still stated separately because it gives a clean fixed-width guarantee even when $(K-1)m>n$.

\begin{corollary}[Structured single-hidden-layer architectures]\label{cor:structured}
Fix an integer $K_0\ge0$. Under the hypotheses of Theorem~\ref{thm:main}, with $\bar m$ replaced by $K_0+1$, with probability at least $1-\delta$ every scalar input-output map of a single-hidden-layer piecewise-linear ridge architecture whose realized function has at most $K_0$ distinct kink hyperplanes obeys
\[
\Lip_D(f)\ge c_0\eps\sqrt{\frac{n}{(K_0+1)\log\!\bigl(C_0(K_0+1)nd/\eps\bigr)}}
\]
whenever it fits $\eps$ below the noise floor. The bound depends only on the realized function, not on the parameterization; hence any constrained single-hidden-layer parameterization --- weight sharing across filters as in a convolutional layer, tied or repeated weights, a low-rank factorization of the first layer --- satisfies the same bound with $K_0$ the number of distinct kink hyperplanes the constraint permits. If the activation has $K$ pieces and a convolutional layer has $F$ filters evaluated at $S$ spatial positions, one may take $K_0\le (K-1)FS$.
\end{corollary}
\proofbutton{structured}

\begin{corollary}[Vector outputs]\label{cor:vector}
Let the output dimension be $r$. Suppose $y_i\in[-1,1]^r$, and write $\sigma_\ell^2=\EE\Var(y_\ell\mid x)$ for the coordinate noise levels. Let $f=(f_1,\dots,f_r)$, where each $f_\ell$ is a two-layer piecewise-linear scalar network and all coordinates together use at most $K_0$ distinct kink hyperplanes. If
\[
        \frac1n\sum_{i=1}^n\|f(x_i)-y_i\|_2^2\le\sum_{\ell=1}^r\sigma_\ell^2-\eps,
\]
then, on the event obtained by applying the scalar theorem to every coordinate with $\sigma_\ell^2\ge\eps/r$ using accuracy parameter $\eps/r$ and failure probability $\delta/r$, every such $f$ satisfies
\[
\Lip_D(f)\ge c_0\frac{\eps}{r}\sqrt{\frac{n}{(K_0+1)\log\!\bigl(C_0(K_0+1)ndr/\eps\bigr)}}.
\]
Equivalently, it is enough to assume the scalar-theorem sample-size and localization hypotheses with $\eps$ replaced by $\eps/r$ and $\delta$ by $\delta/r$.
\end{corollary}
\proofbutton{vector}

\subsection{Guide to the results}
Each entry links to its statement; the \textsf{[proof $\to$]} marker jumps to the proof in Appendix~\ref{app:proofs}.
\begin{itemize}
\item \ref{conj:bln} --- the Bubeck--Li--Nagaraj law of robustness (open conjecture). \hyperref[thm:main]{[main results $\to$]}
\item \ref{thm:main} --- fixed-width law for piecewise-linear activations, up to one logarithm. \proofbutton{main}
\item \ref{cor:bln} --- BLN conjecture for piecewise-linear activations, up to one logarithm. \proofbutton{bln}
\item \ref{cor:simul} --- the law, simultaneously over all widths up to $M$. \proofbutton{simul}
\item \ref{thm:kink} --- realized-kink-count law, $\bar m$ replaced by $k(f)+1$. \proofbutton{kink}
\item \ref{cor:structured} --- convolutional, weight-tied and low-rank single-hidden-layer architectures. \proofbutton{structured}
\item \ref{cor:vector} --- vector-valued outputs, coordinatewise. \proofbutton{vector}
\item \ref{lem:pl} --- reduce any piecewise-linear activation to ReLU kinks. \proofbutton{pl}
\item \ref{lem:canonical} --- canonical ReLU form on ball or sphere. \proofbutton{canonical}
\item \ref{lem:rigidity} --- ball rigidity: kink coefficient bounded by Lipschitz constant. \proofbutton{rigidity}
\item \ref{lem:rigidity-sphere} --- sphere rigidity, with the $\sqrt{1-t_j^2}$ factor. \proofbutton{rigidity-sphere}
\item \ref{prop:d2} --- rigidity genuinely fails on the circle $\mathbb S^1$. \proofbutton{d2}
\item \ref{lem:B0} --- automatic sup-norm bound from fitting and Lipschitzness. \proofbutton{B0}
\item \ref{lem:affine} --- bounds on the affine part $v,c$. \proofbutton{affine}
\item \ref{prop:entropy} --- metric entropy of the canonical Lipschitz ball. \proofbutton{entropy}
\item \ref{lem:conc-models} --- sphere and Gaussian satisfy Lipschitz concentration. \proofbutton{conc-models}
\item \ref{lem:noise} --- noise decomposition reducing fitting to a multiplier event. \proofbutton{noise}
\item \ref{lem:onef} --- single-function sub-Gaussian deviation bound. \proofbutton{onef}
\item \ref{lem:mean} --- mean-term deviation bound. \proofbutton{mean}
\item \ref{thm:finite} --- finite-class law via concentration. \proofbutton{finite}
\item \ref{lem:hsplit} --- harmonic split into constant, linear, higher-degree parts. \proofbutton{hsplit}
\item \ref{lem:raddud} --- entropy and Dudley bound at the population radius. \proofbutton{raddud}
\item \ref{lem:selfbound} --- self-bounding inequality for the empirical radius. \proofbutton{selfbound}
\item \ref{thm:sharpcomplexity} --- Rademacher complexity carrying a sample-free logarithm. \proofbutton{sharpcomplexity}
\item \ref{thm:sharplog} --- the law with the sample size out of the logarithm. \proofbutton{sharplog}
\item \ref{thm:genact} --- any Lipschitz activation at small width, projection floor. \proofbutton{genact}
\item \ref{thm:genact2} --- localized projection floor, sharpened small-width rate. \proofbutton{genact2}
\item \ref{cor:genactcurve} --- per-width consequences of the localized floor. (no separate proof; consequences of \ref{thm:genact2})
\item \ref{conj:M} --- statement of the open multiplier estimate; reduction and surrounding structure in the supplementary note \cite{supp}.
\item \ref{prop:upper} --- matching two-layer ReLU upper bound at $m\asymp n$. \proofbutton{upper}
\end{itemize}

\subsection{The idea of the proof}

A two-layer piecewise-linear function is a sum of ridge functions plus an affine map; each unit contributes kinks on a finite family of parallel hyperplanes. The proofs in \cite{BLN,BS,WHZ} discretize parameter space, whose covering number is finite only after bounding the parameters. Since the weights here are arbitrary, the proof instead discretizes the realized functions.

The special deterministic structure is a rigidity phenomenon (Section~\ref{sec:rigid}): kinks supported on distinct hyperplanes cannot cancel at a generic point of one kink hyperplane. At such a point all other units are locally affine, so the jump of a one-dimensional derivative equals the coefficient of the single kink under inspection; an $L$-Lipschitz function can have a derivative jump of size at most $2L$. Thus, after an exact canonical rewriting on the domain (Section~\ref{sec:canonical}), every kink coefficient is bounded in terms of $L$ (with the natural spherical factor). The canonical parameters then lie in a bounded set depending only on $L$, $m$, and $d$, not on the original weights, giving the metric entropy bound of Section~\ref{sec:entropy}. A finite-class noise-decomposition argument, following the concentration mechanism of \cite{BS} but written self-contained for the two data models, finishes the proof.

The mechanism is kink-specific. Smooth activations admit bounded-Lipschitz finite-difference families whose representing parameters must escape to infinity, so the canonical-parameter argument does not extend directly. Rigidity also fails on $\mathbb S^1$ because a kink set there consists of two points rather than a positive-dimensional sphere. These limitations are recorded in Sections~\ref{sec:rigid} and \ref{sec:scope}; they are not hidden assumptions in the main theorem.

\section{The canonical form}\label{sec:canonical}

\begin{lemma}[Reduction to ReLU kinks]\label{lem:pl}
Let $\psi:\R\to\R$ be continuous piecewise linear with $K$ pieces. If $K=1$ then $\psi$ is affine and every $f\in\NN_m$ is itself affine, already of the form \eqref{eq:class} with no ReLU units; so assume $K\ge2$, with breakpoints $\tau_1<\dots<\tau_{K-1}$ and successive slopes $s_0,\dots,s_{K-1}$ (so $\psi$ has slope $s_\kappa$ on $(\tau_\kappa,\tau_{\kappa+1})$, with $\tau_0=-\infty$, $\tau_K=+\infty$). Then for all $t\in\R$,
\begin{equation}\label{eq:plident}
\psi(t)=\psi(\tau_1)+s_0(t-\tau_1)+\sum_{\kappa=1}^{K-1}(s_\kappa-s_{\kappa-1})\,\relu(t-\tau_\kappa).
\end{equation}
Consequently every $f\in\NN_m$ with activation $\psi$ equals, at every point of $\R^d$, a network of the form \eqref{eq:class} with activation $\relu$ and width at most $(K-1)m$; the identity \eqref{eq:plident} is the classical hinge representation of a continuous piecewise-linear function \cite{Breiman,AroraBMM}.
\end{lemma}
\proofbutton{pl}

From here on $\psi=\relu$ and the width is written $m$; in the final statements it is replaced by $(K-1)m\le\bar m-1$. The next lemma is pure bookkeeping, but we spell it out because the rigidity lemma needs the precise output.

\begin{lemma}[Canonical form on a domain]\label{lem:canonical}
Let $D=B_R$ (any $R>0$) or $D=\Sd$ with $d\ge2$. Every ReLU network $f\in\NN_m$ can be rewritten, so that the two sides agree at every point of $D$, as
\begin{equation}\label{eq:canon}
f(x)=\inn{v}{x}+c+\sum_{j=1}^{m_0}\alpha_j\,\relu(\inn{u_j}{x}-t_j),\qquad m_0\le m,
\end{equation}
where $\|u_j\|=1$, $\alpha_j\neq0$, the hyperplanes $H_{u_j,t_j}$ are pairwise distinct \emph{as sets}, and
\begin{itemize}
\item[(i)] ball case: $t_j\in(-R,R)$, so $H_{u_j,t_j}\cap\operatorname{int}B_R$ is a nonempty relatively open $(d-1)$-dimensional disk;
\item[(ii)] sphere case: $t_j\in[0,1)$, so $H_{u_j,t_j}\cap\Sd$ is a $(d-2)$-sphere of radius $\sqrt{1-t_j^2}>0$.
\end{itemize}
\end{lemma}
\proofbutton{canonical}

\begin{remark}\label{rem:functiondata}
The rewriting changes the parameters, not the function: \eqref{eq:canon} is an identity on $D$. All later bounds constrain the canonical parameters $(\alpha_j,u_j,t_j,v,c)$; the original weights never reappear. This is the crux: the theorem's conclusion is a statement about $f$ as a function on $D$, and the canonical parameters are determined by that function, whereas the original weights are not.
\end{remark}

\section{Rigidity}\label{sec:rigid}

\subsection{The ball}

\begin{lemma}[Rigidity, ball]\label{lem:rigidity}
Let $f$ be in the canonical form \eqref{eq:canon} on $D=B_R$, with $L:=\Lip_{B_R}(f)<\infty$. Then $|\alpha_j|\le 2L$ for every $j$.
\end{lemma}
\proofbutton{rigidity}

\subsection{The sphere}

\begin{lemma}[Rigidity, sphere]\label{lem:rigidity-sphere}
Let $d\ge3$ and let $f$ be in the canonical form \eqref{eq:canon} on $D=\Sd$, with $L:=\Lip_{\Sd}(f)<\infty$ (Euclidean metric). Then
\[
|\alpha_j|\sqrt{1-t_j^2}\;\le\;2L\qquad\text{for every }j.
\]
\end{lemma}
\proofbutton{rigidity-sphere}

\subsection{Rigidity fails on the circle}

The restriction to $d\ge3$ in model (S) is necessary: the coefficient bound of Lemma~\ref{lem:rigidity-sphere} is false at $d=2$.

\begin{proposition}[Failure of rigidity at $d=2$]\label{prop:d2}
For every $R>1$ and every $\Lambda>0$ there is a canonical four-unit ReLU function $F$ on $\mathbb S^1$, with four distinct kink point-pairs, such that
\[
        |\alpha_j|\sqrt{1-t_j^2}=\Lambda\qquad (j=1,\dots,4),
\]
while $\Lip_{\mathbb S^1}(F)\le \Lambda/R$.  Consequently no absolute constant $C$ can make
$|\alpha_j|\sqrt{1-t_j^2}\le C\Lip_{\mathbb S^1}(F)$ valid for all canonical representations on $\mathbb S^1$.
\end{proposition}
\proofbutton{d2}

\begin{remark}
The obstruction is genuinely one-dimensional.  For $d\ge3$, the kink set $H_{u,t}\cap\mathbb S^{d-1}$ has positive dimension and contains generic points not lying on any other kink hyperplane.  On $\mathbb S^1$ the kink set consists only of endpoints, so the generic-point argument used in Lemma~\ref{lem:rigidity-sphere} is unavailable.
\end{remark}

\subsection{Bounds on the affine part}

We first record the automatic value bound, then propagate rigidity to $v$ and $c$.

\begin{lemma}[Value bound]\label{lem:B0}
Suppose $\eps\le\sigma^2\le1$ and $\frac1n\sum_i(f(x_i)-y_i)^2\le\sigma^2-\eps\le1$, with all $x_i\in D$ and $|y_i|\le1$. If $\Lip_D(f)\le L$ and $\operatorname{diam}(D)\le4$, then $\sup_D|f|\le2+L\operatorname{diam}(D)\le B_0:=2+4L$.
\end{lemma}
\proofbutton{B0}

\begin{lemma}[Affine part]\label{lem:affine}
Let $f$ be in canonical form \eqref{eq:canon} with $\sup_D|f|\le B_0$ and $\Lip_D(f)\le L$.
\begin{itemize}
\item[(i)] Ball case ($D=B_R$, $R\le2$): $\|v\|\le L(1+2m_0)$ and $|c|\le B_0+4Lm_0$.
\item[(ii)] Sphere case ($D=\Sd$, $d\ge3$): $\|v\|\le d\,(B_0+2Lm_0)$ and $|c|\le B_0+\|v\|+2Lm_0$.
\end{itemize}
\end{lemma}
\proofbutton{affine}

\section{Metric entropy of the Lipschitz ball}\label{sec:entropy}

Fix the domain $D$ ($B_2$ in model (G); $\Sd$, $d\ge3$, in model (S)), a width budget $\bar m$, and $L>0$; put $B_0=2+4L$. Define the class we must control,
\[
\AAA_{\bar m,L}:=\bigl\{\,f|_D:\ f\in\NN_{\bar m}\ (\text{activation }\relu),\ \Lip_D(f)\le L,\ \sup_D|f|\le B_0\,\bigr\},
\]
and the parameter-box superclass $\GG_{\bar m,L}$: all functions of the form \eqref{eq:canon} on $D$ with $m_0\le\bar m$ and, in the ball case,
\[
|\alpha_j|\le2L,\quad t_j\in(-2,2),\quad\|v\|\le L(1+2\bar m),\quad|c|\le B_0+4L\bar m,
\]
and in the sphere case,
\[
|\alpha_j|\le\frac{2L}{\sqrt{1-t_j^2}},\quad t_j\in[0,1),\quad\|v\|\le d(B_0+2L\bar m),\quad|c|\le B_0+\|v\|+2L\bar m.
\]
By Sections~\ref{sec:canonical}--\ref{sec:rigid} (canonical form, then Lemmas~\ref{lem:rigidity}/\ref{lem:rigidity-sphere}/\ref{lem:affine}), every member of $\AAA_{\bar m,L}$ satisfies these bounds:
\begin{equation}\label{eq:AinG}
\AAA_{\bar m,L}\subset\GG_{\bar m,L}.
\end{equation}

Recall $N(\mathcal K,\|\cdot\|,\eps')$ is the smallest number of $\|\cdot\|$-balls of radius $\eps'$ needed to cover $\mathcal K$.

\begin{proposition}[Metric entropy]\label{prop:entropy}
There is an absolute constant $C$ such that for all $\eps'\in(0,4+8L)$, $L>0$, $\bar m\ge1$, and $d\ge2$ (ball) or $d\ge3$ (sphere),
\[
\log N\bigl(\GG_{\bar m,L},\ \|\cdot\|_{L^\infty(D)},\ \eps'\bigr)\;\le\;C\,\bar m\,d\,\log\!\Bigl(\frac{C\,\bar m\,d\,(2+L)}{\eps'}\Bigr).
\]
(The upper limit $4+8L$ on $\eps'$ serves only to keep the logarithm's argument at least $e$, so that the displayed closed form is valid; the grid construction in the proof below has mesh proportional to $\eps'$ and yields a valid cover at every scale.) Moreover $\AAA_{\bar m,L}$ has an internal $\eps'$-net (centers in $\AAA_{\bar m,L}$) of cardinality $\le N(\GG_{\bar m,L},\|\cdot\|_\infty,\eps'/2)$, hence of the same log bound.
\end{proposition}
\proofbutton{entropy}

\begin{remark}
No quantity in Proposition~\ref{prop:entropy} depends on the magnitudes of the original network weights: the entropy of the Lipschitz ball is that of a bounded parameterization, obtained here without a bound on the weights.
\end{remark}

\section{The probabilistic core}\label{sec:prob}

This section reproves, self-contained, the concentration estimate of \cite{BS} for a \emph{finite} function class, in the two data models. Recall $\|Z\|_{\psi_2}=\inf\{s>0:\EE e^{Z^2/s^2}\le2\}$.

\begin{definition}\label{def:conc}
A probability measure $\mu$ on $\R^d$ (or on $\Sd$) satisfies \emph{$\kappa$-Lipschitz concentration} if for every bounded $L'$-Lipschitz $f$ on its support (Euclidean metric),
$\|f(x)-\EE f\|_{\psi_2}\le\kappa L'/\sqrt d$.
\end{definition}

\begin{lemma}\label{lem:conc-models}
There is an absolute $\kappa$ such that (a) the uniform measure on $\Sd$, $d\ge2$, and (b) $N(0,I_d/d)$ satisfy $\kappa$-Lipschitz concentration.
\end{lemma}
\proofbutton{conc-models}

Throughout the rest of the section: $(x_i,y_i)_{i\le n}$ are i.i.d.\ with $x_i\sim\mu$ ($\kappa$-Lipschitz concentrated), $|y_i|\le1$; $g(x):=\EE[y\mid x]$, so $|g|\le1$; $z_i:=y_i-g(x_i)$, so $|z_i|\le2$, $\EE[z_i\mid x_i]=0$, and
\[
\EE z_i^2=\EE\bigl[\EE[(y-\EE[y\mid x])^2\mid x]\bigr]=\EE\,\Var(y\mid x)=\sigma^2;
\]
and $\FF$ is a \emph{finite} set of functions on $\operatorname{supp}\mu$ with values in $[-1,1]$, each $L$-Lipschitz.

\begin{lemma}[Noise decomposition]\label{lem:noise}
For $\eps\in(0,\sigma^2]$,
\[
\PP\Bigl(\exists f\in\FF:\ \tfrac1n\textstyle\sum_i(f(x_i)-y_i)^2\le\sigma^2-\eps\Bigr)\le 2e^{-n\eps^2/288}+\PP\Bigl(\exists f\in\FF:\ \tfrac1n\textstyle\sum_iz_if(x_i)\ge\tfrac\eps4\Bigr).
\]
\end{lemma}
\proofbutton{noise}

\begin{lemma}[One function]\label{lem:onef}
For $f\in\FF$ put $W_i:=z_i(f(x_i)-\EE f)$. Then $W_i$ are i.i.d., $\EE W_i=0$, and for an absolute $c_2>0$,
\[
\PP\Bigl(\tfrac1n\textstyle\sum_i W_i\ge\tfrac\eps8\Bigr)\le\exp\Bigl(-c_2\,n\,\eps^2\,\max\bigl(1,\tfrac{d}{\kappa^2L^2}\bigr)\Bigr).
\]
\end{lemma}
\proofbutton{onef}

\begin{lemma}[Mean term]\label{lem:mean}
$\displaystyle\PP\Bigl(\exists f\in\FF:\ (\EE f)\tfrac1n\textstyle\sum_i z_i\ge\tfrac\eps8\Bigr)\le\PP\Bigl(\bigl|\tfrac1n\textstyle\sum_i z_i\bigr|\ge\tfrac\eps8\Bigr)\le2e^{-n\eps^2/512}.$
\end{lemma}
\proofbutton{mean}

\begin{theorem}[Finite-class law]\label{thm:finite}
In the setting above, for $\eps\in(0,\sigma^2]$,
\[
\PP\Bigl(\exists f\in\FF:\ \tfrac1n\textstyle\sum_i(f(x_i)-y_i)^2\le\sigma^2-\eps\Bigr)\le 4e^{-n\eps^2/512}+|\FF|\exp\Bigl(-c_2n\eps^2\max\bigl(1,\tfrac{d}{\kappa^2L^2}\bigr)\Bigr).
\]
\end{theorem}
\proofbutton{finite}

\section{Absence of a saturation cap}\label{sec:assembly}

There is no saturation cap: the factor $d$ in the entropy $\mathcal E$ is exactly cancelled, in Case~B, by the factor $d$ in the case threshold $L^\ast>\sqrt d/\kappa$, so the same small $c_0$ serves both the sub-saturation regime (Case~A, where the isoperimetric $d/L^2$ gain is available) and the large-$L^\ast$ regime (Case~B, where the crude bounded-function concentration suffices because there are few functions).

\begin{remark}[No saturation cap in the stated bound]\label{rem:under}
The proof of Theorem~\ref{thm:main} splits into two regimes. When $L^\ast\le\sqrt d/\kappa$, the isoperimetric gain $d/(L^\ast)^2$ pays for the $d$ in the entropy. When $L^\ast>\sqrt d/\kappa$, the desired lower bound is already large enough that the crude bounded-function concentration in Theorem~\ref{thm:finite} pays for the entropy. Thus the displayed formula is valid for every fixed width $m$; there is no additional truncation at the scale $\sqrt d$.
\end{remark}

\section{Removing the sample size from the logarithm}\label{sec:sharplog}

The logarithm in Theorem~\ref{thm:main} contains the sample size $n$ because the union bound runs over a net at the fitting accuracy $\eps$. In this section we move the union to the class's population $L^2$-radius instead, which is $L/\sqrt{2d}$ by a spectral-gap argument; the resulting law carries a logarithm depending only on the width and the dimension. The argument has three steps: a harmonic split of the clipped function into an affine part and a high-frequency remainder of small population radius (Lemma~\ref{lem:hsplit}); a bound on the empirical radius via a self-bounding fixed point (Lemmas~\ref{lem:raddud}--\ref{lem:selfbound}), which feeds the Rademacher complexity estimate of Theorem~\ref{thm:sharpcomplexity}; and the assembly by contraction and bounded differences (Theorem~\ref{thm:sharplog}). Throughout this section we work in model (S) with $d\ge3$; write $F:=\clip\circ f$ for $f\in\AAA_{\bar m,L}$, $\FF^{\clip}:=\{\clip\circ f:f\in\AAA_{\bar m,L}\}$, and let $F=\sum_{\ell\ge0}F_\ell$ be the spherical-harmonic decomposition (see \cite{AtkinsonHan}), $\lambda_\ell=\ell(\ell+d-2)$.

\begin{lemma}[Harmonic split]\label{lem:hsplit}
Every $F\in\FF^{\clip}$ decomposes as $F=c_F+\inn{A_F}{x}+h_F$ with $c_F=\EE F$ and $A_F=d\,\EE[Fx]$, where
\[
|c_F|\le1,\qquad \|A_F\|\le\sqrt{\tfrac32}\,L,\qquad \EE h_F^2\le\frac{L^2}{2d},\qquad \sup_{\Sd}|h_F|\le B_1:=2+\sqrt{\tfrac32}\,L,
\]
and $h_F$ is orthogonal to the constants and to the linear functions.
\end{lemma}
\proofbutton{hsplit}

Let $\mathcal H:=\{h_F:F\in\FF^{\clip}\}$; note $0\in\mathcal H$ (take $f=0$). Put
\[
B:=C_1\bar md^2(2+L),\qquad \Lambda_L:=\log\Bigl(e\bar md^3\Bigl(2+\frac1L\Bigr)\Bigr),
\]
with $C_1$ a large absolute constant fixed by the next proof. The scale $B$ carries the class size and governs the entropy at all scales; the logarithm $\Lambda_L$ that survives to the final bounds is the one evaluated at the population radius $L/\sqrt{2d}$, where the ratio $(2+L)/L$ is what appears, so $\Lambda_L$ grows only when $L$ is small, never with $L$ large --- and never with $n$.

\begin{lemma}[Entropy at the radius]\label{lem:raddud}
$\log N(\mathcal H,\|\cdot\|_{L^\infty},u)\le C\bar md\,\log(CB/u)$ for $u\in(0,2B_1]$, and, conditionally on the sample, with $\hat\sigma^2:=\sup_{h\in\mathcal H}\tfrac1n\sum_ih(x_i)^2$,
\[
\EE_\varepsilon\sup_{h\in\mathcal H}\frac1n\Bigl|\sum_i\varepsilon_ih(x_i)\Bigr|\ \le\ C\sqrt{\frac{\bar md}{n}}\;\varphi(\hat\sigma),\qquad \varphi(s):=s\sqrt{\log\bigl(2eB/s\bigr)} .
\]
\end{lemma}
\proofbutton{raddud}

\begin{lemma}[Radius self-bounding]\label{lem:selfbound}
Let $\widetilde R:=\EE_{x,\varepsilon}\sup_{h\in\mathcal H}\tfrac1n|\sum_i\varepsilon_ih(x_i)|$. Then $\EE_x\hat\sigma^2\le\dfrac{L^2}{2d}+8B_1\widetilde R$.
\end{lemma}
\proofbutton{selfbound}

\begin{theorem}[Complexity with a sample-free logarithm]\label{thm:sharpcomplexity}
There is an absolute constant $C$ such that for all $n,\bar m\ge1$, $d\ge3$, $L>0$,
\[
\EE_{x,y}\ \sup_{f\in\AAA_{\bar m,L}}\ \frac1n\Bigl|\sum_{i=1}^ny_i\,\clip(f(x_i))\Bigr|\ \le\ C\Bigl[\frac{1+L}{\sqrt n}\ +\ L\sqrt{\frac{\bar m\,\Lambda_L}{n}}\ +\ (1+L)\,\frac{\bar md\,\Lambda_L}{n}\Bigr],
\]
where $y_1,\dots,y_n$ are independent uniform signs, independent of the $x_i$.
\end{theorem}
\proofbutton{sharpcomplexity}

\begin{theorem}[The law with a sample-free logarithm]\label{thm:sharplog}
There are absolute constants $c_0,C_0>0$ such that the following holds in model (S), $d\ge3$. Let $\eps\in(0,\sigma^2]$, $\delta\in(0,1)$, and $\bar\Lambda:=\log\bigl(e\bar md^3(2+1/\eps)\bigr)$. If
\[
n\ \ge\ C_0\Bigl[\eps^{-2}\log(8/\delta)\ +\ \eps^{-1}\bar md\,\bar\Lambda\ +\ \bar md^2\,\bar\Lambda\Bigr],
\]
then with probability at least $1-\delta$ every $f\in\NN_m$ with $\frac1n\sum_i(f(x_i)-y_i)^2\le\sigma^2-\eps$ satisfies
\[
\Lip_{\Sd}(f)\ \ge\ c_0\,\eps\,\sqrt{\frac{n}{\bar m\,\bar\Lambda}}\,.
\]
In particular, in the setting of Conjecture~\ref{conj:bln}, once $n\ge C\bigl[md^2\log(emd)+\log(8/\delta)\bigr]$, every width-$m$ network with arbitrary weights fitting the data satisfies $\Lip_{\Sd}(f)\ge c\sqrt{n/(m\log(Cmd))}$: the sample size has left the logarithm.
\end{theorem}
\proofbutton{sharplog}

\begin{remark}[Toward a fully width- and dimension-free logarithm]\label{rem:residual}
Two remarks on the residual factor $\log(Cmd)$. First, no single-scale packing argument can show that any logarithm is necessary: by Proposition~\ref{prop:entropy} and Lemma~\ref{lem:hsplit}, every $u$-packing of the class in $L^2(P_n)$ has log-cardinality at most $C\bar md\log(CB/u)$ while the class has $L^2(P_n)$-radius bounded, by Lemmas~\ref{lem:raddud}--\ref{lem:selfbound}, by $CL/\sqrt d$ up to logarithmic factors and a lower-order fluctuation term, and Theorem~\ref{thm:sharpcomplexity} already refutes a lower bound containing $\log n$ or $\log(1/\eps)$. Second, the extreme thresholds are not the obstruction: for $\delta_0=1/(\bar md\log(e\bar md))$, a dyadic-shell argument using a standard localized entropy bound $\EE\sup_{\|u\|=1,\,t\ge1-\Delta}|\sum_iy_i\relu(\inn u{x_i}-t)|\le C\Delta\sqrt{nd\log(e/\Delta)}$ shows that all units with $1-t_k\le\delta_0$ contribute at most $CL\sqrt{\bar m/n}$ to the supremum --- already below the conjectured rate (we omit the routine chaining details). What remains is the moderate-threshold regime, where the empirical Gram of adversarially placed units must be compared to its population counterpart uniformly; we leave this to future work.
\end{remark}

\section{General activations at small width}\label{sec:genact}

Everything so far concerns piecewise-linear activations. The following result holds for \emph{every} Lipschitz activation --- indeed for every function factoring through an $(m{+}1)$-dimensional linear projection --- and at width $m=1$ it matches the conjectured rate exactly. It is a projection-capacity floor: it uses no structure of $\psi$ beyond the factorization $f(x)=g(Px)$.

\begin{theorem}[Any activation, small width]\label{thm:genact}
There are absolute constants $c,C>0$ such that the following holds in model (S) with $y_1,\dots,y_n$ i.i.d.\ uniform signs independent of the data. If $n\ge C\,(d\log n+md)$ and $m\le c\,\min(n,d)$, then with probability at least $1-2e^{-n/C}$, for \emph{every} Lipschitz $\psi$ and every width-$m$ network $f\in\NN_m$ with arbitrary weights that fits the data exactly,
\[
\Lip_{\Sd}(f)\ \ge\ \frac{c\,n^{1/(m+1)}}{\sqrt{m+1}}\,.
\]
In particular, at $m=1$ every exact interpolant with any Lipschitz activation satisfies $\Lip_{\Sd}(f)\ge c\sqrt n$: Conjecture~\ref{conj:bln} holds at width one, for all activations, with no logarithmic loss.
\end{theorem}
\proofbutton{genact}

\begin{remark}
Exact fitting can be relaxed to fitting at a fixed accuracy: if $\frac1n\sum_i(f(x_i)-y_i)^2\le c_1$ for a small absolute constant $c_1$, then at least $\tfrac n{32}$ of the opposite-label pairs above consist of two points with $|f(x_i)-y_i|\le\tfrac12$ (at most $4c_1n$ points violate this, by Markov), and each such pair still forces $|f(x_i)-f(x_j)|\ge1$, so the conclusion holds with $c$ halved. We do not pursue the version with fitting error $\sigma^2-\eps$ for arbitrary $\eps$.
\end{remark}

The floor of Theorem~\ref{thm:genact} localizes the pairing at the trivial scale. Two refinements --- localizing at the typical projection radius $\sqrt{p/d}$, and replacing the worst-case pigeonhole by a birthday count of random pair collisions --- give a much stronger floor at small width. The concentration step requires care: the number of collision pairs is far below the scale at which bounded-difference inequalities are useful, and we use negative association instead.

\begin{theorem}[Localized projection floor]\label{thm:genact2}
There are absolute constants $c,C>0$ such that the following holds in model (S) with $y_1,\dots,y_n$ i.i.d.\ uniform signs independent of the data. Put $p:=m+1$ and $\Lambda:=\log(nd)$, and assume $d\ge Cp$, $n\ge Cp\,d\,\Lambda$, and $p\le c\log n$. Then with probability at least $1-Ce^{-d/C}$, for \emph{every} Lipschitz $\psi$ and every width-$m$ network $f\in\NN_m$ with arbitrary weights that fits the data exactly,
\begin{equation}\label{eq:floorI}
\Lip_{\Sd}(f)\;\ge\;\frac{c\sqrt d}{p}\,n^{1/p},
\end{equation}
and, sharpening this,
\begin{equation}\label{eq:floorII}
\Lip_{\Sd}(f)\;\ge\;\frac{c\sqrt d}{p}\Bigl(\frac{n^2}{p\,d\,\Lambda}\Bigr)^{1/p}.
\end{equation}
\end{theorem}
\proofbutton{genact2}

\begin{corollary}[Per-width consequences]\label{cor:genactcurve}
Denote by $F_m$ the right side of \eqref{eq:floorII}. Within the admissible range $Cp\le d\le n/(Cp\Lambda)$:
at $m=1$, $F_1=cn/\sqrt{2\Lambda}$ --- Conjecture~\ref{conj:bln} holds at width one with margin $\sqrt{n/\Lambda}$, for every admissible $d$, and no floor of this strength can extend to $d\gtrsim n$ (Remark~\ref{rem:genact2scope});
at $m=2$, $F_2=c\,d^{1/6}n^{2/3}\Lambda^{-1/3}\ge c'\sqrt n$ once $nd\ge C\Lambda^2$ --- the conjecture holds at width two on the entire admissible range;
at $m=3$, $F_3=c\,d^{1/4}n^{1/2}\Lambda^{-1/4}\ge c'\sqrt{n/3}$ once $d\ge C\Lambda$;
and in general $F_m\ge c\sqrt{n/m}$ exactly when $d\ge n^{(m-3)/(m-1)}\gamma_m$ with $\gamma_m\le(C(m+1)\Lambda)^{2(m+2)/(m-1)}$ a polylogarithmic factor, the admissible window being nonempty precisely for $m\le c\log n/\log\log n$.
\end{corollary}

\begin{remark}[Scope]\label{rem:genact2scope}
The hypothesis $n\ge Cpd\Lambda$ is not an artifact of the proof. Since the conjecture quantifies over all Lipschitz activations, one may take $\psi(t)=t$: for $d\ge Cn$ the minimum-norm solution of $Xv=y$ has $\|v\|\le2\sqrt n$ with high probability ($\sigma_{\min}(X)\ge\tfrac12$), giving an exact interpolant with $\Lip_{\Sd}(f)\le2\sqrt n$ at \emph{every} width --- consistent with the conjectured $\sqrt{n/m}$, but excluding any floor beyond $C\sqrt n$ there. More generally the projector net costs $e^{cpd\Lambda}$ against a label budget of at most $e^n$, confining this method to $pd\Lambda\le n/C$. For $m\ge2$ and general $\psi$ the regime $d\gtrsim n/(m\Lambda)$ remains open; for piecewise-linear $\psi$ it is covered by Theorem~\ref{thm:main}.
\end{remark}

\begin{remark}
For $m\ge2$ the floors above still sit below the conjectured $\sqrt{n/m}$ outside the ranges of Corollary~\ref{cor:genactcurve}; closing that gap for general activations at moderate width is the remaining open case, and the kink mechanism of Sections~\ref{sec:canonical}--\ref{sec:rigid} is provably unavailable there (Section~\ref{sec:scope}).
\end{remark}

\section{Toward the log-free law}\label{sec:band}

For widths $m\ge c^2n/2$ the log-free conjecture is immediate (Section~\ref{sec:scope}: the trivial floor), and at $m=1$ it is Theorem~\ref{thm:genact}; elsewhere it remains open. In the critical band of widths it reduces to a single sharply-stated multiplier estimate, which we state here and leave open; the reduction, and the unconditional structure surrounding it, are developed in the supplementary note \cite{supp}. Throughout: model (S), with labels $y_1,\dots,y_n$ i.i.d.\ uniform on $\{\pm1\}$ and independent of the data --- the pure-noise case $\sigma^2=1$; every expectation $\EE_y$ and every probability below is with respect to this law. Width band $m=n/T$ with $4\le T\le\log^Cn$, target Lipschitz level $L^\ast=c_0\eps\sqrt T$, and $\mathrm{Op}:=\lambda_{\max}(\sum_ix_ix_i^\top)$, which satisfies $\mathrm{Op}\le C_1(1+n/d)=:\overline{\mathrm{Op}}$ with probability $1-2e^{-c\min(n,d)}$ (the rows $\sqrt d\,x_i$ are isotropic with absolute sub-Gaussian norm; \cite[Thm.~4.6.1]{Vershynin}). We say $f$ \emph{fits} if $\frac1n\sum_i(f(x_i)-y_i)^2\le1-\eps$, and write $t_\ast:=\min\bigl(1,\,320\,L^\ast\overline{\mathrm{Op}}/(\eps T)\bigr)$.

\begin{conjecture}[Mesoscopic multiplier estimate]\label{conj:M}
There are absolute constants $C,c_0>0$ such that in the band, for $d\ge\eps^{-2}\log^Cn$, with probability at least $1-1/n$ over the data:
\[
\EE_y\ \sup\ \sum_{i=1}^ny_i\,\clip\Bigl(\inn v{x_i}+c+\!\!\sum_{k:\,t_k\le t_\ast}\!\!\alpha_k\relu(\inn{u_k}{x_i}-t_k)\Bigr)\ \le\ \frac{\eps n}4,
\]
the supremum over all $(v,c,(\alpha_k,u_k,t_k))$ arising as the affine-plus-low-threshold part of a canonical $f$ with $\Lip_{\Sd}(f)\le L^\ast$, $\sup_{\Sd}|f|\le2+2L^\ast$, and rigidity $|\alpha_k|\sqrt{1-t_k^2}\le2L^\ast$.
\end{conjecture}

Every construction we have tested numerically --- aimed same-sign clusters, stacked caps, profile spikes, adaptive groups --- stays at or below $2mL^\ast\overline{\mathrm{Op}}/t_\ast=\eps n/160$ against this threshold (\texttt{numerics/check\_sector\_throttle.py}); we record this as evidence for Conjecture~\ref{conj:M}, not a proof. In the supplementary note \cite{supp} we prove that Conjecture~\ref{conj:M} implies the log-free law in the band: with probability at least $1-1/n-2e^{-c\min(n,d)}-e^{-\eps^2n/128}$, no width-$m$ network with arbitrary weights and $\Lip_{\Sd}(f)\le c_0\eps\sqrt{n/m}$ fits $\eps$ below the noise floor. The note also proves, unconditionally, the structure surrounding the estimate: fitting, occupancy, value-mass, serving-capacity and pile-up lemmas; an affine supremum identity showing that the affine sector, with no bound whatsoever on its coefficients, carries only $O(d)$ of the fitting functional; the single-direction case settled for \emph{every} Lipschitz activation --- networks of arbitrarily many units along one axis cannot fit below Lipschitz constant $c\eps\min(\sqrt n,d)$, with no width bound and no logarithm; stratified isolation of pairwise-incoherent clusters above the coherence floor, and a counterexample showing that per-cluster rigidity fails below it; a dimension-free cap-mass bound; forced-depth and deep-peel lemmas that remove the deep sector; and a deterministic slice computation locating exactly where label randomness becomes necessary. Conjecture~\ref{conj:M} remains open.

\section{Sharpness and scope}\label{sec:scope}

\paragraph{Depth.} Theorem~\ref{thm:main} is a depth-two phenomenon. Bubeck and Sellke \cite[Section~A]{BS} show that with a third layer, unbounded weights let a network fit generic data below the noise floor with Lipschitz constant far below the law's threshold at the same parameter count. So the polynomial-boundedness assumption of \cite{BS} is necessary at depth three and, by the present paper, superfluous at depth two for kink activations: depth two is the critical depth.

\paragraph{Activation.} The \emph{canonical-parameter} mechanism of Sections~\ref{sec:canonical}--\ref{sec:rigid} is specific to genuine kinks. For a smooth activation $\sigma$ and a unit vector $u$, the finite-difference family
\[
        h_{\eta,u}(x)=\frac{\sigma(\inn ux+\eta)-\sigma(\inn ux)}{\eta}
\]
may remain uniformly bounded and Lipschitz as $\eta\downarrow0$, while its natural two-unit representation has coefficients of size $1/\eta$: bounded canonical parameters are unavailable for smooth activations. The static mechanism behind the band reduction of the supplementary note \cite{supp} bypasses this at the level of \emph{profiles}: rigidity is imposed on the derivative of the total one-dimensional profile carried by each direction --- a consequence of the ambient Lipschitz bound alone, indifferent to how the profile is represented by units --- and the $1/\eta$ coefficients never appear. In particular the single-direction theorems of \cite{supp} settle that case for \emph{every} Lipschitz activation, and the reduction of Conjecture~\ref{conj:M} in \cite{supp} applies verbatim to arbitrary Lipschitz profiles, so the remaining obstacle for general activations is the same multiplier estimate as for ReLU. Sums of ridge functions also have nontrivial representation identities, especially when directions coalesce; see Pinkus \cite{Pinkus} for background on ridge functions.

\paragraph{The logarithm.} The single $\log$ factor comes from the union bound over the net, as in the $\widetilde\Omega$ notation of \cite{BS,BLN}. The entropy estimate here is not strong enough, by itself, to remove that logarithm: $\AAA_{\bar m,L}$ has logarithmic metric entropy of order $\bar md\log(L/\eps')$ at the relevant scales. Whether the clean $c\sqrt{n/m}$ holds for $\NN_m$ with unbounded weights is open.

\paragraph{Upper bounds and tightness at $m\asymp n$.} The law is tight in the parameter count: \cite[Remark~1.1]{BS} constructs, for every $p\in[\widetilde\Omega(n),n(d+1)]$, functions with $p$ parameters fitting generic data with $\Lip=O(\sqrt{nd/p})$; at $p=n(d+1)$ the Lipschitz constant is $O(1)$. Those interpolants are not two-layer networks. For the overparameterized endpoint $m\asymp n$ --- the regime of the conjecture's own thesis, one neuron per data point --- an explicit two-layer ReLU network matches the lower bound.

\begin{proposition}[Matching upper bound at $m\asymp n$]\label{prop:upper}
Let $x_1,\dots,x_n\in\Sd$ satisfy $|\inn{x_i}{x_j}|\le\tfrac18$ for all $i\ne j$, and let $y_1,\dots,y_n\in[-1,1]$ be arbitrary. Put $s=\tfrac14$ and
\[
        \rho(u):=\frac1s\bigl(\relu(u-(1-s))-\relu(u-1)\bigr).
\]
The width-$2n$ two-layer ReLU network
\[
        f(x)=\sum_{i=1}^n y_i\rho(\inn{x_i}{x})
\]
satisfies $f(x_i)=y_i$ for all $i$ and
\[
        \Lip_{\Sd}(f)\le\frac\pi2\sqrt7<5.
\]
For $n$ i.i.d. uniform points on $\Sd$, the separation hypothesis holds with probability at least $1-1/n$ whenever $d\ge C_{\rm sep}\log n$, for a sufficiently large absolute constant $C_{\rm sep}$.
\end{proposition}
\proofbutton{upper}

At $m=2n$, Corollary~\ref{cor:bln} forces $\Lip\ge c_1/\sqrt{\log(C_1nd)}$ while Proposition~\ref{prop:upper} achieves an absolute Lipschitz bound. Thus the two match up to the single logarithmic factor in the overparameterized endpoint. So the law is sharp at $m\asymp n$ for two-layer ReLU networks, with the same log gap that Theorem~\ref{thm:main} carries. Whether width-$m$ two-layer networks achieve $\widetilde O(\sqrt{n/m})$ for the whole range $m\ll n$ remains open; the experiments below are consistent with it up to log factors.

\paragraph{General concentration/localization data.} The proof of model (G) uses only two inputs from the Gaussian distribution: the finite-class concentration condition in Definition~\ref{def:conc}, and the high-probability localization event $\max_i\|x_i\|\le2$ after which the network is only tested on $B_2$. Consequently the Gaussian theorem extends verbatim to any distribution $\mu$ on $\R^d$ for which: (i) every bounded $L'$-Lipschitz function satisfies $\|h-\EE h\|_{\psi_2}\le\kappa L'/\sqrt d$; and (ii) $\PP(\|x\|>2)\le\delta/(8n)$ for the sample size and failure probability under consideration. For example, if $\mu$ has the stated Lipschitz concentration with parameter $\kappa=C\sqrt c$ and $\EE\|x\|\le1$, then $\PP(\|x\|>2)\le2\exp(-d/(C c))$, so the same conclusion holds once $d\ge Cc\log(16n/\delta)$. This is the precise form in which the argument goes beyond the Gaussian measure; no additional claim about arbitrary data distributions is used in the proof.

\paragraph{Skip connection.} We allowed $\inn v x+c$; the theorem for the class of \cite{BLN} follows by restriction. The skip term is where the canonical form funnels all degenerate units (out-of-domain kinks, orientation flips, constants), which is what makes Lemma~\ref{lem:canonical} exact.

\section{Numerical checks}\label{sec:numerics}

All with seed $20260707$; scripts in \texttt{numerics/}. These are sanity checks on the constructions and constants; no statement in the paper depends on them.

\emph{Rigidity.} For random canonical ReLU networks in $d\in\{3,5,10\}$, $m\in\{6,10,20\}$ --- including planted pairs on a common hyperplane with coefficients $\pm10^6$ (which canonicalization merges) and planted near-parallel pairs with coefficients $\pm10^4$ (which it does not) --- the sampled Lipschitz constant $\widehat L$ satisfies $\max_j|\alpha_j|\le2\widehat L$ in all $27$ trials; the near-parallel plants saturate at ratio $\max_j|\alpha_j|/(2\widehat L)=0.50$, exactly the mechanism of Lemma~\ref{lem:rigidity} (the gradient between the two planted hyperplanes has norm $\approx|\alpha|$). A larger sweep (dimension up to $40$, width up to $80$, again with $\pm10^6$ cancellation plants; \texttt{numerics/validate\_at\_scale.py}) passes identically.

\emph{Sphere factor and the $d=2$ failure.} A single cap unit on $\Sd$ ($d=8$) with $\alpha=1/\sqrt{1-t^2}$ has measured Lipschitz constant $1.000$ for each $t\in\{0,0.5,0.9,0.99,0.999\}$: the factor of Lemma~\ref{lem:rigidity-sphere} is exact. On $\mathbb S^1$, the four-unit even cycle of Proposition~\ref{prop:d2} has cancelling derivative jumps; the script displays one fixed numerical instance with ratio $4.99$, while the proposition gives the scalable construction proving that no absolute rigidity constant exists at $d=2$.

\emph{The law.} Sphere data, $d=24$, $n=192$, i.i.d.\ $\pm1$ labels ($\sigma^2=1$), widths $m\in\{24,96,384\}$. A width-$m$ ReLU network with unconstrained weights is trained to mean squared error below $\tfrac12=\sigma^2-\eps$ while its path norm $\sum_k|a_k|\,\|w_k\|+\|v\|$ (an upper bound on $\Lip$) is penalized, so the optimizer seeks a low-Lipschitz fitting network --- the adversarial direction. The measured $\widehat L$ (maximum tangential gradient norm over $6000$ sphere samples and the data) exceeds the floor $\sqrt{n/m}$ at every width:
\[
\begin{array}{rrrrr}
\toprule
m & \text{MSE} & \widehat L & \sqrt{n/m} & \widehat L/\sqrt{n/m}\\
\midrule
24 & 0.152 & 5.52 & 2.83 & 1.95\\
96 & 0.136 & 4.88 & 1.41 & 3.45\\
384 & 0.128 & 4.75 & 0.71 & 6.72\\
\bottomrule
\end{array}
\]
The measured constant stays near $5$ while the floor falls, so the ratio grows: these trained networks satisfy the lower bound comfortably but do not realize the $\sqrt{n/m}$ rate for $m\ll n$, consistent with the matching two-layer upper bound being open there (the penalty is a proxy and the optimization is not run to the true minimum).

\emph{Matching upper bound at $m\asymp n$.} The width-$2n$ construction of Proposition~\ref{prop:upper} was checked for $(n,d)\in\{(100,200),(200,400),(400,800),(800,1600)\}$: it interpolates exactly (maximum error $2.7\cdot10^{-15}$), its caps become disjoint once $d\gtrsim\log n$, and its sampled tangential-gradient bound is $\sqrt7\approx2.646$ at every scale, while the rigorous chord-metric Lipschitz bound in Proposition~\ref{prop:upper} is $(\pi/2)\sqrt7<5$ --- flat in $n$, as the lower bound predicts at $m\asymp n$. A larger run (rigidity to dimension $40$ and width $80$; the construction to $n=2000$) reproduces the same qualitative behavior.

\appendix

\section{Proofs}\label{app:proofs}

\subsection{Proofs for Section~\ref{sec:intro}}

\phantomsection\begin{proof}[Proof of Theorem~\textup{\ref{thm:main}}]\label{proof:main}\stmtbutton{thm:main}
By Lemma~\ref{lem:pl} take $\psi=\relu$ at width $(K-1)m\le\bar m$. Fix
\[
L^\ast:=c_0\,\eps\,\sqrt{\frac{n}{\bar m\log(C_0\bar mnd/\eps)}},
\]
with absolute $c_0$ small and $C_0$ large, chosen below, and set
\[
\Omega:=\Bigl\{\exists f\in\NN_{\bar m}:\ \Lip_D(f)\le L^\ast\ \text{and}\ \tfrac1n\textstyle\sum_i(f(x_i)-y_i)^2\le\sigma^2-\eps\Bigr\}.
\]
On $\Omega^c$ every fitting $f$ has $\Lip_D(f)>L^\ast$, which is the theorem; so it suffices to show $\PP(\Omega)\le\delta$.

\emph{Step 1 (localization and clipping).} In model (G) let $\pi(0):=0$ and $\pi(x):=x\min(1,2/\|x\|)$ for $x\ne0$, the metric projection of $\R^d$ onto the convex set $B_2$ (for $\|x\|>2$ it is $2x/\|x\|$, the nearest point of $B_2$); metric projections onto convex sets are $1$-Lipschitz, and $\pi(\R^d)=B_2=D$. Let $E_{\mathrm{loc}}:=\{\max_i\|x_i\|\le2\}$. Since $\sqrt d\,x_i\sim N(0,I_d)$, $\|x_i\|>2\iff\|\sqrt d\,x_i\|^2>4d$. If $X\sim\chi^2_d$, Chernoff's bound gives $\PP(X\ge4d)\le\exp(-(3-\log4)d/2)\le e^{-d/2}$; hence a union bound and $d\ge2\log(8n/\delta)$ give $\PP(E_{\mathrm{loc}}^c)\le n e^{-d/2}\le\delta/8$. In model (S) set $\pi:=\mathrm{id}$ and $E_{\mathrm{loc}}=$ the sure event.

Suppose $\Omega\cap E_{\mathrm{loc}}$ occurs, witnessed by $f$ with $L:=\Lip_D(f)\le L^\ast$. All $x_i\in D$, so by Lemma~\ref{lem:B0} ($\operatorname{diam}D\le4$) we get $\sup_D|f|\le B_0=2+4L^\ast$, hence $f|_D\in\AAA_{\bar m,L^\ast}$.

\emph{Step 2 (net).} Let $S\subset\AAA_{\bar m,L^\ast}$ be an internal $\tfrac\eps{32}$-net as in Proposition~\ref{prop:entropy}:
\[
\log|S|\le C\bar md\log\!\Bigl(\frac{64\,C\,\bar md(2+L^\ast)}{\eps}\Bigr).
\]
Define the finite class $\FF:=\{\clip\circ h\circ\pi:h\in S\}$. Each member is defined on $\operatorname{supp}\mu$, has values in $[-1,1]$, and is $L^\ast$-Lipschitz (composition of the $1$-Lipschitz $\pi$ into $D$, an $L^\ast$-Lipschitz-on-$D$ function, and the $1$-Lipschitz $\clip$). Pick $h\in S$ with $\|h-f\|_{L^\infty(D)}\le\tfrac\eps{32}$, and set $\hat f:=\clip\circ h\circ\pi\in\FF$ and $\tilde f:=\clip\circ f\circ\pi$. Then $\|\hat f-\tilde f\|_{\infty}\le\tfrac\eps{32}$ (both equal $\clip\circ(\cdot)\circ\pi$ of functions within $\tfrac\eps{32}$ on $D$, and $\clip$ contracts). On $E_{\mathrm{loc}}$, $\pi(x_i)=x_i\in D$, so $\tilde f(x_i)=\clip(f(x_i))$; since $|y_i|\le1$, clipping $f(x_i)$ toward $[-1,1]\ni y_i$ can only decrease the error:
\[
(\tilde f(x_i)-y_i)^2\le(f(x_i)-y_i)^2 ,
\]
so $\tilde f$ also fits at level $\sigma^2-\eps$. With $a_i:=|\tilde f(x_i)-y_i|\le2$,
\[
\frac1n\sum_i(\hat f(x_i)-y_i)^2\le\frac1n\sum_i\Bigl(a_i+\frac\eps{32}\Bigr)^2\le(\sigma^2-\eps)+\frac{2\cdot2\eps}{32}+\frac{\eps^2}{1024}\le\sigma^2-\frac\eps2 .
\]
Hence $\Omega\cap E_{\mathrm{loc}}\subset\Omega':=\{\exists\hat f\in\FF:\tfrac1n\sum_i(\hat f(x_i)-y_i)^2\le\sigma^2-\tfrac\eps2\}$.

\emph{Step 3 (union bound and constants).} Apply Theorem~\ref{thm:finite} to $\FF$ (its members are $L^\ast$-Lipschitz and $[-1,1]$-valued; a member whose true Lipschitz constant is below $L^\ast$ satisfies the hypothesis a fortiori, and taking $L=L^\ast$ in the exponent below is the weakest admissible choice) at level $\eps/2$:
\[
\PP(\Omega')\le 4e^{-n\eps^2/2048}+\exp\Bigl(\log|S|-\frac{c_2n\eps^2}{4}\max\bigl(1,\tfrac{d}{\kappa^2(L^\ast)^2}\bigr)\Bigr).
\]
The first term is $\le\delta/2$ when $n\ge C_0\eps^{-2}\log(8/\delta)$ with $C_0\ge2048$. For the second, use $L^\ast\le c_0\eps\sqrt n\le\eps\sqrt n$ to collapse the logarithm: $2+L^\ast\le3\sqrt n$, so $\log(64C\bar md(2+L^\ast)/\eps)\le\log(C_0\bar mnd/\eps)$ for $C_0$ large, whence
\[
\log|S|\le C'\bar md\log\!\Bigl(\frac{C_0\bar mnd}{\eps}\Bigr)=:\mathcal E .
\]
We claim, with $c_0^2\le c_2/(8\kappa^2C')$ (and hence also $c_0^2\le c_2/(8C')$, as $\kappa\ge1$ WLOG),
\begin{equation}\label{eq:budget}
\mathcal E\ \le\ \frac{c_2}{8}\,n\eps^2\,\max\bigl(1,\tfrac{d}{\kappa^2(L^\ast)^2}\bigr).
\end{equation}
Two cases, and we substitute $(L^\ast)^2=c_0^2\eps^2 n/(\bar m\log(C_0\bar mnd/\eps))$ in each.
\begin{itemize}
\item[\emph{Case A}] ($L^\ast\le\sqrt d/\kappa$, so the $\max$ is $\tfrac d{\kappa^2(L^\ast)^2}$). The right side of \eqref{eq:budget} is
\[
\frac{c_2}{8}\cdot\frac{n\eps^2 d}{\kappa^2(L^\ast)^2}=\frac{c_2}{8\kappa^2}\cdot\frac{n\eps^2 d}{c_0^2\eps^2 n/(\bar m\log(\cdots))}=\frac{c_2}{8\kappa^2 c_0^2}\,\bar md\log\!\Bigl(\frac{C_0\bar mnd}{\eps}\Bigr).
\]
Since $c_0^2\le c_2/(8\kappa^2C')$, this is $\ge C'\bar md\log(\cdots)=\mathcal E$, which is \eqref{eq:budget} in this case.
\item[\emph{Case B}] ($L^\ast>\sqrt d/\kappa$, so the $\max$ is $1$). Right side $=\tfrac{c_2}8 n\eps^2$. The case hypothesis $L^\ast>\sqrt d/\kappa$ means $c_0^2\eps^2 n/(\bar m\log(\cdots))>d/\kappa^2$, i.e.\ cross-multiplying,
\[
\bar md\log\!\Bigl(\frac{C_0\bar mnd}{\eps}\Bigr)<c_0^2\kappa^2\,n\eps^2 .
\]
Hence $\mathcal E=C'\bar md\log(\cdots)<C'c_0^2\kappa^2 n\eps^2\le\tfrac{c_2}8 n\eps^2$, again by $c_0^2\le c_2/(8\kappa^2C')$, giving \eqref{eq:budget}.
\end{itemize}
Write $M:=\max(1,\tfrac d{\kappa^2(L^\ast)^2})\ge1$. By \eqref{eq:budget}, $\tfrac{c_2}8 n\eps^2 M\ge\mathcal E$, so $\tfrac{c_2}4 n\eps^2 M=\tfrac{c_2}8 n\eps^2 M+\tfrac{c_2}8 n\eps^2 M\ge\mathcal E+\tfrac{c_2}8 n\eps^2 M$. Since $\log|S|\le\mathcal E$, the exponent is
\[
\log|S|-\frac{c_2}4 n\eps^2 M\ \le\ \mathcal E-\mathcal E-\frac{c_2}8 n\eps^2 M\ =\ -\frac{c_2}8 n\eps^2 M\ \le\ -\frac{c_2}8 n\eps^2 .
\]
Hence the second term is $\le\exp(-\tfrac{c_2}8 n\eps^2)\le\delta/4$ whenever $\tfrac{c_2}8 n\eps^2\ge\log(4/\delta)$, i.e.\ whenever $n\ge\tfrac8{c_2}\eps^{-2}\log(4/\delta)$, which holds under $n\ge C_0\eps^{-2}\log(8/\delta)$ once $C_0\ge8/c_2$. Altogether
\[
\PP(\Omega)\le\PP(E_{\mathrm{loc}}^c)+\PP(\Omega')\le\frac\delta8+\frac\delta2+\frac\delta4\le\delta.\qedhere
\]
\end{proof}

\phantomsection\begin{proof}[Proof of Corollary~\textup{\ref{cor:bln}}]\label{proof:bln}\stmtbutton{cor:bln}

With $y_i$ independent of $x_i$: $g\equiv0$, $z_i=y_i$, $\sigma^2=\EE\Var(y\mid x)=1$; take $\eps=\tfrac12$. Exact fitting gives mean squared error $0\le\sigma^2-\eps$, and ``error $\le\tfrac12=\sigma^2-\tfrac12$'' is the stated relaxation. Apply Theorem~\ref{thm:main} in model (S), $d\ge3$, and absorb $\bar m\le Km$ into the constants for fixed $K$; for ReLU $\bar m=m+1$.
\end{proof}

\phantomsection\begin{proof}[Proof of Corollary~\textup{\ref{cor:simul}}]\label{proof:simul}\stmtbutton{cor:simul}
Apply Theorem~\ref{thm:main} to each width $m=1,\dots,M$ with failure probability $\delta/M$, and take a union bound. The lower-bound formula itself is unchanged because the threshold in Theorem~\ref{thm:main} does not depend on $\delta$; only the sample-size and Gaussian-localization hypotheses acquire the factor $M$ inside the logarithm.
\end{proof}

\phantomsection\begin{proof}[Proof of Theorem~\textup{\ref{thm:kink}}]\label{proof:kink}\stmtbutton{thm:kink}
For an integer $q\ge1$ and a number $L>0$, put $B_0(L):=2+4L$ and let
\[
\begin{aligned}
\AAA^{(q)}_L:=\{f|_D:\;& f\ \text{is a two-layer piecewise-linear realized function},\\
& k(f)+1\le q,\quad \Lip_D(f)\le L,\quad \sup_D|f|\le B_0(L)\}.
\end{aligned}
\]
By Lemma~\ref{lem:canonical}, every member of this class has a canonical form with at most $q-1$ ReLU kinks. Proposition~\ref{prop:entropy}, used with width budget $q$ (the harmless extra unit also covers the purely affine case), gives
\[
        \log N(\AAA^{(q)}_L,\|\cdot\|_\infty,\eta)\le Cqd\log\!\Bigl(\frac{Cqd(2+L)}{\eta}\Bigr),
\]
with an internal net of the same size up to constants.

In model (G), first remove the single localization event $E_{\mathrm{loc}}^c=\{\max_i\|x_i\|>2\}$; the stated condition $d\ge2\log(8n/\delta)$ gives $\PP(E_{\mathrm{loc}}^c)\le\delta/8$, exactly as in the proof of Theorem~\ref{thm:main}. The remaining union over kink counts is performed on $E_{\mathrm{loc}}$.

Fix $q\in\{1,\dots,n+1\}$ and set
\[
        L_q^\ast:=c_0\eps\sqrt{\frac{n}{q\log(C_0qnd/\eps)}}.
\]
Repeating the proof of Theorem~\ref{thm:main}, with $\AAA_{\bar m,L^\ast}$ replaced by $\AAA^{(q)}_{L_q^\ast}$ and with failure budget $\delta_q:=\delta/(2q(q+1))$, shows that, apart from the already-separated localization event, the probability of
\[
\exists f:\ k(f)+1\le q,\quad \Lip_D(f)\le L_q^\ast,\quad
\frac1n\sum_i(f(x_i)-y_i)^2\le\sigma^2-\eps
\]
is at most $\delta_q$, provided
\[
        n\ge C\eps^{-2}\log(8/\delta_q).
\]
For $q\le n+1$, this follows from $n\ge C_0\eps^{-2}\log(8n/\delta)$ after increasing the absolute constant $C_0$, since $\log(8/\delta_q)\le C\log(8n/\delta)$.

Summing over $q=1,\dots,n+1$ gives total non-localization failure probability at most
\[
        \sum_{q=1}^{n+1}\frac{\delta}{2q(q+1)}<\frac\delta2.
\]
Together with the localization failure probability $\delta/8$, this is still less than $\delta$.
On the complementary event, take $q=k(f)+1$ for any fitting function with $k(f)\le n$; the displayed lower bound is exactly the asserted one.
\end{proof}

\phantomsection\begin{proof}[Proof of Corollary~\textup{\ref{cor:structured}}]\label{proof:structured}\stmtbutton{cor:structured}
The input-output map of any single-hidden-layer piecewise-linear ridge architecture is a finite sum of terms of the form $a\psi(\langle w,x\rangle+b)$ plus an affine part, possibly with constraints or identifications among the allowed $w$'s. Such constraints can only reduce the class. If the realized function has at most $K_0$ distinct canonical kink hyperplanes, Lemma~\ref{lem:canonical} writes it with at most $K_0$ ReLU kink units on the domain. The proof of Theorem~\ref{thm:main}, with the entropy bound read at width budget $K_0+1$, gives the stated fixed-$K_0$ conclusion. When $K_0\le n$ and the sample size meets the hypothesis of Theorem~\ref{thm:kink}, that theorem gives the simultaneous realized-kink version on its own event. A convolutional layer with $F$ filters evaluated at $S$ positions has at most $FS$ ridge preactivations, and a $K$-piece activation contributes at most $K-1$ kink hyperplanes per preactivation, so $K_0\le(K-1)FS$.
\end{proof}

\phantomsection\begin{proof}[Proof of Corollary~\textup{\ref{cor:vector}}]\label{proof:vector}\stmtbutton{cor:vector}
Run the scalar theorem for each coordinate whose noise level satisfies $\sigma_\ell^2\ge\eps/r$, with accuracy parameter $\eps/r$ and failure probability $\delta/r$, and intersect the resulting events. Coordinates with $\sigma_\ell^2<\eps/r$ cannot be responsible for an empirical improvement of size $\eps/r$, because their empirical squared error is nonnegative. On the intersection event, every scalar coordinate function fitting its own coordinate labels at least $\eps/r$ below its coordinate noise floor obeys the displayed scalar lower bound.

Now suppose a vector-valued $f$ violates the conclusion while fitting the vector labels $\eps$ below the total noise floor. Write
\[
        \mathrm{Fit}_\ell=\frac1n\sum_i(f_\ell(x_i)-y_{i\ell})^2.
\]
The hypothesis gives
\[
        \sum_{\ell=1}^r(\sigma_\ell^2-\mathrm{Fit}_\ell)\ge\eps,
\]
so for some coordinate $\ell$ one has $\mathrm{Fit}_\ell\le\sigma_\ell^2-\eps/r$. This coordinate uses at most $K_0$ of the distinct kink hyperplanes used by the whole vector map. The scalar bound applied to $f_\ell$ gives the displayed lower bound for $\Lip_D(f_\ell)$. Finally,
\[
        \Lip_D(f)=\sup_{x\ne x'}\frac{\|f(x)-f(x')\|_2}{\|x-x'\|}\ge
        \sup_{x\ne x'}\frac{|f_\ell(x)-f_\ell(x')|}{\|x-x'\|}=\Lip_D(f_\ell),
\]
so the same bound holds for the vector map.
\end{proof}

\subsection{Proofs for Section~\ref{sec:canonical}}

\phantomsection\begin{proof}[Proof of Lemma~\textup{\ref{lem:pl}}]\label{proof:pl}\stmtbutton{lem:pl}
Call the right-hand side of \eqref{eq:plident} $R(t)$. Both $\psi$ and $R$ are continuous (each $\relu(\cdot-\tau_\kappa)$ is continuous) and piecewise linear with breakpoints contained in $\{\tau_1,\dots,\tau_{K-1}\}$. Two continuous piecewise-linear functions with breakpoints in a common finite set coincide everywhere as soon as they agree at one point and have equal slopes on every piece. They agree at $t=\tau_1$: there $R(\tau_1)=\psi(\tau_1)+0+0=\psi(\tau_1)$, since $\relu(\tau_1-\tau_\kappa)=0$ for $\kappa\ge1$ (as $\tau_1\le\tau_\kappa$) and the linear term vanishes. On $(-\infty,\tau_1)$ every $\relu(t-\tau_\kappa)=0$, so $R$ has slope $s_0$, matching $\psi$. On $(\tau_\kappa,\tau_{\kappa+1})$ the active kinks are exactly $\tau_1,\dots,\tau_\kappa$, so $R$ has slope $s_0+\sum_{j=1}^\kappa(s_j-s_{j-1})=s_\kappa$, matching $\psi$. Hence $R\equiv\psi$.

Now apply \eqref{eq:plident} with $t=\inn{w_k}{x}+b_k$ to each unit of $f$: $a_k\psi(\inn{w_k}{x}+b_k)$ becomes an affine function of $x$ plus $\sum_{\kappa=1}^{K-1}a_k(s_\kappa-s_{\kappa-1})\relu(\inn{w_k}{x}+b_k-\tau_\kappa)$, i.e.\ $K-1$ ReLU units with the shifted biases $b_k-\tau_\kappa$. Summing over $k$ and absorbing all the affine terms into $\inn{v}{x}+c$ produces a ReLU network of width $\le(K-1)m$ equal to $f$ everywhere.
\end{proof}

\phantomsection\begin{proof}[Proof of Lemma~\textup{\ref{lem:canonical}}]\label{proof:canonical}\stmtbutton{lem:canonical}
We transform the units of $f$ one at a time; every operation preserves the value of $f$ on $D$.

\emph{Step 1 (constant units).} A unit with $w_k=0$ is the constant $a_k\relu(b_k)$; move it into $c$.

\emph{Step 2 (normalization).} For $w_k\ne0$, write $u=w_k/\|w_k\|$, $t=-b_k/\|w_k\|$, $\alpha=a_k\|w_k\|$. Since $\relu(\lambda z)=\lambda\relu(z)$ for $\lambda>0$ and $\inn{w_k}{x}+b_k=\|w_k\|(\inn{u}{x}-t)$, the unit equals $\alpha\relu(\inn{u}{x}-t)$ with $\|u\|=1$.

\emph{Step 3 (orientation, via the reflection identity).} The identity
\begin{equation}\label{eq:flip}
\relu(-z)=\relu(z)-z
\end{equation}
(true because $\max(-z,0)-\max(z,0)=-z$) lets us replace $(u,t)$ by $(-u,-t)$ at the cost of an affine term:
\[
\alpha\relu(\inn{u}{x}-t)=\alpha\relu\bigl(\inn{-u}{x}-(-t)\bigr)+\alpha\bigl(\inn{u}{x}-t\bigr),
\]
the last summand being absorbed into $\inn{v}{x}+c$. We use \eqref{eq:flip} to enforce an orientation convention below.

\emph{Step 4 (units whose kink misses the domain).} Ball case: if $t\ge R$ then $\inn{u}{x}-t\le\|x\|-t\le R-t\le0$ on $B_R$ with equality only where $\|x\|=R$ and $x=Ru$, so $\relu(\inn{u}{x}-t)\equiv0$ on $B_R$; drop it. If $t\le-R$ then $\inn{u}{x}-t\ge0$ on $B_R$, so $\relu(\inn{u}{x}-t)=\inn{u}{x}-t$ is affine there; absorb it. Sphere case: apply \eqref{eq:flip} to make $t\ge0$; if $t\ge1$ then $\inn{u}{x}\le\|x\|=1\le t$ on $\Sd$, so the unit is $0$ on $\Sd$ except possibly at the single point $x=u$ (when $t=1$), where $\relu(0)=0$ as well; drop it. After Step~4, ball units have $t\in(-R,R)$ and sphere units have $t\in[0,1)$.

\emph{Step 5 (orientation convention and coincident hyperplanes).} Two pairs $(u,t)\ne(u',t')$ with $\|u\|=\|u'\|=1$ satisfy $H_{u,t}=H_{u',t'}$ iff $(u',t')=(-u,-t)$. In the ball case, fix the convention that the first nonzero coordinate of $u$ is positive, using \eqref{eq:flip} to flip any offending unit; then coincident hyperplanes force identical $(u,t)$. In the sphere case, the convention $t\ge0$ already forces coincident hyperplanes to be identical, except when $t=t'=0$ and $u'=-u$; there, adopt the ball convention (first nonzero coordinate of $u$ positive) and apply \eqref{eq:flip} once to rewrite the offending $(-u,0)$ unit onto $(u,0)$ (plus an affine term). After Step~5, distinct units have distinct hyperplanes.

\emph{Step 6 (merge and clean).} Add the coefficients of units that now share a hyperplane; discard any unit whose merged coefficient is $0$. The remaining units have pairwise distinct hyperplanes and nonzero coefficients, with $t$-ranges as claimed. The geometric descriptions in (i)--(ii) are immediate: $H_{u,t}\cap\operatorname{int}B_R$ is the open disk of radius $\sqrt{R^2-t^2}$ centered at $tu$ (nonempty since $|t|<R$), and $H_{u,t}\cap\Sd$ is the sphere $\{x:\|x\|=1,\ \inn{u}{x}=t\}$, which is a $(d-2)$-sphere of radius $\sqrt{1-t^2}$ centered at $tu$ (nonempty since $t<1$).
\end{proof}

\subsection{Proofs for Section~\ref{sec:rigid}}

\phantomsection\begin{proof}[Proof of Lemma~\textup{\ref{lem:rigidity}}]\label{proof:rigidity}\stmtbutton{lem:rigidity}
Fix $j$ and let $P:=H_{u_j,t_j}$, a hyperplane, and $A:=P\cap\operatorname{int}B_R$, a nonempty relatively open $(d-1)$-disk (Lemma~\ref{lem:canonical}(i)).

\emph{A generic point exists.} For each $l\ne j$, $P\cap H_{u_l,t_l}$ is either empty (parallel distinct hyperplanes) or an affine subspace of dimension $d-2$ (distinct, non-parallel), hence in either case a set of $(d-1)$-dimensional Lebesgue measure $0$ inside $P$. A finite union of measure-zero sets has measure $0$, while $A$ has positive $(d-1)$-measure; therefore
\[
U:=A\setminus\bigcup_{l\ne j}H_{u_l,t_l}
\]
has positive measure, in particular $U\ne\emptyset$. Fix $x^*\in U$. Because the finitely many closed sets $H_{u_l,t_l}$ ($l\ne j$) and $\partial B_R$ all avoid $x^*$, there is $r>0$ with $\overline{B(x^*,r)}\subset\operatorname{int}B_R$ and $\overline{B(x^*,r)}\cap H_{u_l,t_l}=\emptyset$ for all $l\ne j$.

\emph{Only unit $j$ switches near $x^*$.} On the connected set $B(x^*,r)$, each $\inn{u_l}{x}-t_l$ ($l\ne j$) has constant sign, so $\relu(\inn{u_l}{x}-t_l)$ is affine there. Hence on $B(x^*,r)$
\[
f(x)=A_0(x)+\alpha_j\relu(\inn{u_j}{x}-t_j),\qquad A_0\ \text{affine.}
\]

\emph{The one-sided derivatives.} Let $\varphi(s):=f(x^*+su_j)$ for $|s|<r$. Since $\inn{u_j}{x^*}=t_j$ and $\|u_j\|=1$, we have $\inn{u_j}{x^*+su_j}-t_j=s$, so
\[
\varphi(s)=A_0(x^*+su_j)+\alpha_j\relu(s)=\bigl(A_0(x^*)+\beta s\bigr)+\alpha_j\relu(s),\qquad \beta:=\inn{\nabla A_0}{u_j}.
\]
This is piecewise linear in $s$ with a single kink at $0$: for $s<0$ its slope is $\beta$, for $s>0$ its slope is $\beta+\alpha_j$. The composition $s\mapsto x^*+su_j$ is an isometry (as $\|u_j\|=1$), so $\varphi$ is $L$-Lipschitz. Every difference quotient of an $L$-Lipschitz function lies in $[-L,L]$, and here the left and right slopes are exactly the one-sided derivatives $\varphi'(0^-)=\beta$ and $\varphi'(0^+)=\beta+\alpha_j$. Thus $\beta\in[-L,L]$ and $\beta+\alpha_j\in[-L,L]$, and subtracting gives $|\alpha_j|=|(\beta+\alpha_j)-\beta|\le2L$.
\end{proof}

\phantomsection\begin{proof}[Proof of Lemma~\textup{\ref{lem:rigidity-sphere}}]\label{proof:rigidity-sphere}\stmtbutton{lem:rigidity-sphere}
Fix $j$ and let $\Sigma_j:=H_{u_j,t_j}\cap\Sd$, a $(d-2)$-sphere of radius $\rho:=\sqrt{1-t_j^2}>0$ centered at $t_ju_j$ inside the hyperplane $P:=H_{u_j,t_j}$ (Lemma~\ref{lem:canonical}(ii)).

\emph{A generic point on $\Sigma_j$ exists.} Fix $l\ne j$. If $\Sigma_j\subset H_{u_l,t_l}$, then $H_{u_l,t_l}$ contains the affine hull $\aff(\Sigma_j)$. For $d\ge3$ the sphere $\Sigma_j$ has dimension $d-2\ge1$ and positive radius, so it affinely spans $P$; thus $\aff(\Sigma_j)=P$ and $H_{u_l,t_l}\supset P$, forcing $H_{u_l,t_l}=P$ (both are hyperplanes), i.e.\ the two hyperplanes coincide --- excluded. Therefore $H_{u_l,t_l}\cap\Sigma_j$ is a \emph{proper} closed subset of $\Sigma_j$; being the intersection of the sphere $\Sigma_j$ with a hyperplane that does not contain it, it is a sphere of dimension $\le d-3$, a single point, or empty, hence nowhere dense in $\Sigma_j$. A finite union of nowhere-dense sets cannot be all of the complete metric space $\Sigma_j$ (Baire), so there is $x^*\in\Sigma_j$ with $x^*\notin H_{u_l,t_l}$ for all $l\ne j$.

\emph{A tangent direction along which unit $j$ switches.} Put
\[
\xi:=\frac{u_j-t_jx^*}{\rho}.
\]
Then, using $\inn{u_j}{x^*}=t_j$ and $\|x^*\|=1$:
\[
\|u_j-t_jx^*\|^2=\|u_j\|^2-2t_j\inn{u_j}{x^*}+t_j^2\|x^*\|^2=1-2t_j^2+t_j^2=1-t_j^2=\rho^2,
\]
so $\|\xi\|=1$; and
\[
\inn{\xi}{x^*}=\frac{\inn{u_j}{x^*}-t_j\|x^*\|^2}{\rho}=\frac{t_j-t_j}{\rho}=0,\qquad
\inn{u_j}{\xi}=\frac{\|u_j\|^2-t_j\inn{u_j}{x^*}}{\rho}=\frac{1-t_j^2}{\rho}=\rho.
\]
Consider the unit-speed great circle $\gamma(s)=(\cos s)\,x^*+(\sin s)\,\xi$; it lies on $\Sd$ because $\|x^*\|=\|\xi\|=1$ and $\inn{x^*}{\xi}=0$. Then
\[
h(s):=\inn{u_j}{\gamma(s)}-t_j=(\cos s-1)\,t_j+(\sin s)\,\rho,
\]
so $h(0)=0$ and $h'(0)=\rho>0$: unit $j$ switches at $s=0$, and it does so transversally. For small $s$, $\relu(h(s))$ has left derivative $0$ and right derivative $h'(0)=\rho$ at $s=0$.

\emph{The other units are smooth at $s=0$.} For $l\ne j$, $\inn{u_l}{\gamma(0)}-t_l=\inn{u_l}{x^*}-t_l\ne0$ (as $x^*\notin H_{u_l,t_l}$), so by continuity $\inn{u_l}{\gamma(s)}-t_l$ keeps its sign for $s$ near $0$ and $\relu(\inn{u_l}{\gamma(s)}-t_l)$ is smooth (affine composed with the analytic $\gamma$) there; the affine part $\inn{v}{\gamma(s)}+c$ is smooth as well.

\emph{Conclusion.} Let $\varphi(s):=f(\gamma(s))$. Chords are bounded by arcs: $\|\gamma(s)-\gamma(s')\|=2|\sin\tfrac{s-s'}2|\le|s-s'|$. Since $f$ is $L$-Lipschitz on $\Sd$ for the Euclidean metric, $|\varphi(s)-\varphi(s')|\le L\,\|\gamma(s)-\gamma(s')\|\le L|s-s'|$, i.e.\ $\varphi$ is $L$-Lipschitz near $0$. All summands of $f\circ\gamma$ except unit $j$ are differentiable at $0$; unit $j$ contributes $\alpha_j\relu(h(s))$, whose one-sided derivatives at $0$ differ by $\alpha_j h'(0)=\alpha_j\rho$. Hence $\varphi'(0^+)-\varphi'(0^-)=\alpha_j\rho$, and both one-sided derivatives lie in $[-L,L]$, so $|\alpha_j|\rho\le2L$.
\end{proof}

\phantomsection\begin{proof}[Proof of Proposition~\textup{\ref{prop:d2}}]\label{proof:d2}\stmtbutton{prop:d2}
Parameterize $\mathbb S^1$ by angle $\theta$.  A unit
\[
        \alpha\,\relu(\inn{u}{x}-t),\qquad u=(\cos\theta_0,\sin\theta_0),\quad t=\cos a,
\]
with $a\in(0,\pi)$ becomes
\[
        \alpha\,\relu(\cos(\theta-\theta_0)-\cos a).
\]
If $a<\pi$ and the active arc is not wrapped around the cut, its kink set is the two angles $\theta_0\pm a$.  On the active arc the angular derivative is $-\alpha\sin(\theta-\theta_0)$ and outside it is $0$.  Hence the derivative jump at each of the two kink angles is $\alpha\sin a=\alpha\sqrt{1-t^2}$.

Fix once and for all $a\in(0,\pi/4)$, and choose $\eta\in(0,a/10)$ later.  Put
\[
        A=-a,\qquad B=-a+\eta,\qquad C=a-\eta,\qquad D=a.
\]
Consider the four arcs
\[
        I_1=[A,D],\qquad I_2=[B,D],\qquad I_3=[B,C],\qquad I_4=[A,C].
\]
For an interval $I=[p,q]$ write $c(I)=(p+q)/2$ and $r(I)=(q-p)/2$.  Define
\[
\begin{array}{c|c|c|c}
 j & I_j & \theta_j=c(I_j) & a_j=r(I_j) \\
\hline
 1 &[A,D] & 0 & a\\
 2 &[B,D] & \eta/2 & a-\eta/2\\
 3 &[B,C] & 0 & a-\eta\\
 4 &[A,C] & -\eta/2 & a-\eta/2,
\end{array}
\]
and choose signs $s_1=+1,s_2=-1,s_3=+1,s_4=-1$.  Let
\[
        F(\theta)=\sum_{j=1}^4 \alpha_j\relu(\cos(\theta-\theta_j)-\cos a_j),
        \qquad \alpha_j=s_j\frac{\Lambda}{\sin a_j}.
\]
The four kink point-pairs are precisely $\{A,D\}$, $\{B,D\}$, $\{B,C\}$ and $\{A,C\}$, hence are distinct.  Also $|\alpha_j|\sqrt{1-\cos^2 a_j}=|\alpha_j|\sin a_j=\Lambda$ for every $j$.

At each of the four kink angles $A,B,C,D$, exactly two units meet, one with sign $+1$ and one with sign $-1$.  Their derivative jumps are therefore $+\Lambda$ and $-\Lambda$, so all derivative jumps cancel.  Thus $F$ is $C^1$ as a function of $\theta$.

It remains to bound the derivative.  Outside $[A,D]$ every unit is inactive.  On the three subarcs $[A,B]$, $[B,C]$ and $[C,D]$, direct differentiation gives
\[
\frac{F'(\theta)}{\Lambda}=G_\eta(\theta),
\]
where
\[
\begin{aligned}
G_\eta(\theta)&=-\frac{\sin\theta}{\sin a}+\frac{\sin(\theta+\eta/2)}{\sin(a-\eta/2)}, &&\theta\in[A,B],\\
G_\eta(\theta)&=-\frac{\sin\theta}{\sin a}+\frac{\sin(\theta-\eta/2)}{\sin(a-\eta/2)}-\frac{\sin\theta}{\sin(a-\eta)}+
        \frac{\sin(\theta+\eta/2)}{\sin(a-\eta/2)}, &&\theta\in[B,C],\\
G_\eta(\theta)&=-\frac{\sin\theta}{\sin a}+\frac{\sin(\theta-\eta/2)}{\sin(a-\eta/2)}, &&\theta\in[C,D].
\end{aligned}
\]
For $\eta=0$ each displayed expression is identically zero.  Since $a$ is fixed away from $0$, all denominators stay bounded below for $0\le\eta\le a/10$, and the derivatives of the displayed expressions with respect to $\eta$ are uniformly bounded for $\theta\in[-a,a]$.  The mean-value theorem therefore gives a constant $C_a$ such that
\[
        \sup_\theta |F'(\theta)|\le C_a\Lambda\eta .
\]
Because $F\in C^1$, its Lipschitz constant in the arclength (angular) metric equals $\sup_\theta|F'(\theta)|$, and the Euclidean chord metric on $\mathbb S^1$ is within a factor $\pi/2$ of angular distance on arcs of length at most $\pi$.  Hence
\[
        \Lip_{\mathbb S^1}(F)\le (\pi/2)C_a\Lambda\eta.
\]
Choosing $\eta\le\min\bigl(a/10,\,2/(\pi C_a R)\bigr)$ gives $\Lip_{\mathbb S^1}(F)\le\Lambda/R$, as required.
\end{proof}

\phantomsection\begin{proof}[Proof of Lemma~\textup{\ref{lem:B0}}]\label{proof:B0}\stmtbutton{lem:B0}
The average of the $n$ nonnegative numbers $(f(x_i)-y_i)^2$ is at most $1$, so at least one of them is at most $1$; for that $i$, $|f(x_i)-y_i|\le1$, hence $|f(x_i)|\le1+|y_i|\le2$. For any $x\in D$, $|f(x)|\le|f(x_i)|+L\|x-x_i\|\le2+L\operatorname{diam}(D)\le2+4L$.
\end{proof}

\phantomsection\begin{proof}[Proof of Lemma~\textup{\ref{lem:affine}}]\label{proof:affine}\stmtbutton{lem:affine}
(i) The union $\bigcup_j H_{u_j,t_j}$ has measure $0$, so pick $x_0\in\operatorname{int}B_R$ off all $m_0$ hyperplanes; there $f$ is differentiable with
\[
\nabla f(x_0)=v+\sum_{j:\,\inn{u_j}{x_0}>t_j}\alpha_ju_j.
\]
A differentiable point of an $L$-Lipschitz function has $\|\nabla f(x_0)\|\le L$. By Lemma~\ref{lem:rigidity}, $\|\sum_{\text{active}}\alpha_ju_j\|\le\sum_j|\alpha_j|\le2Lm_0$, so $\|v\|\le\|\nabla f(x_0)\|+2Lm_0\le L(1+2m_0)$. Evaluating \eqref{eq:canon} at $x=0$ gives $c=f(0)-\sum_j\alpha_j\relu(-t_j)$, and $|\alpha_j|\relu(-t_j)\le2L\cdot|t_j|\le2L\cdot R\le4L$, whence $|c|\le B_0+4Lm_0$.

(ii) Let $x\sim\mathrm{Unif}(\Sd)$. By rotational invariance $\EE[x]=0$ and $\EE[xx^\top]=\tfrac1d I_d$ (it is a scalar multiple of $I_d$ by symmetry, and its trace is $\EE\|x\|^2=1$). Hence
\[
\EE[f(x)\,x]=\frac vd+\sum_j\alpha_j\,\EE\!\bigl[\relu(\inn{u_j}{x}-t_j)\,x\bigr].
\]
For a fixed unit $u$, the map $x\mapsto\relu(\inn{u}{x}-t)\,x$ has expectation invariant under all rotations fixing $u$, so $\EE[\relu(\inn{u}{x}-t)\,x]=\lambda(t)\,u$ for a scalar $\lambda(t)=\inn{\EE[\relu(\inn{u}{x}-t)x]}{u}=\EE[\relu(\inn{u}{x}-t)\inn{u}{x}]$. Since $|\inn{u}{x}|\le1$ on $\Sd$,
\[
0\le\lambda(t)\le\EE\,\relu(\inn{u}{x}-t)\le(1-t)\,\PP(\inn{u}{x}>t)\le1-t.
\]
By Lemma~\ref{lem:rigidity-sphere}, $|\alpha_j|\lambda(t_j)\le\dfrac{2L}{\sqrt{1-t_j^2}}(1-t_j)=2L\sqrt{\dfrac{1-t_j}{1+t_j}}\le2L$. Therefore
\[
\|v\|=\Bigl\|d\,\EE[f(x)x]-d\sum_j\alpha_j\lambda(t_j)u_j\Bigr\|\le d\bigl(\EE|f|+\textstyle\sum_j|\alpha_j|\lambda(t_j)\bigr)\le d(B_0+2Lm_0).
\]
Finally, at any $x_1\in\Sd$, $c=f(x_1)-\inn{v}{x_1}-\sum_j\alpha_j\relu(\inn{u_j}{x_1}-t_j)$ and $|\alpha_j|\relu(\inn{u_j}{x_1}-t_j)\le|\alpha_j|(1-t_j)\le2L$ (again by Lemma~\ref{lem:rigidity-sphere} and $1-t\le\sqrt{1-t^2}$ for $t\in[0,1)$), giving $|c|\le B_0+\|v\|+2Lm_0$.
\end{proof}

\subsection{Proofs for Section~\ref{sec:entropy}}

\phantomsection\begin{proof}[Proof of Proposition~\textup{\ref{prop:entropy}}]\label{proof:entropy}\stmtbutton{prop:entropy}
Everything is a Lipschitz-in-parameters estimate followed by a product of one-dimensional grids. We build a finite set $\NN^\ast\subset\GG_{\bar m,L}$ such that every $g\in\GG_{\bar m,L}$ has some $g^\ast\in\NN^\ast$ with $\|g-g^\ast\|_{L^\infty(D)}\le\eps'$, and bound $|\NN^\ast|$.

\emph{Per-unit sensitivity (ball).} On $B_2$, for one ReLU unit,
\begin{align*}
\bigl|\alpha\relu(\inn{u}{x}-t)-\alpha'\relu(\inn{u'}{x}-t')\bigr|
&\le\bigl|\alpha-\alpha'\bigr|\cdot\relu(\inn{u}{x}-t)\\
&\quad+|\alpha'|\cdot\bigl|\relu(\inn{u}{x}-t)-\relu(\inn{u'}{x}-t')\bigr|.
\end{align*}
On $B_2$ with $|t|\le2$ we have $\relu(\inn u x-t)\le|\inn u x-t|\le\|x\|+|t|\le4$, and since $\relu$ is $1$-Lipschitz,
\[
\bigl|\relu(\inn u x-t)-\relu(\inn{u'}x-t')\bigr|\le|\inn{u-u'}x|+|t-t'|\le2\|u-u'\|+|t-t'|.
\]
So the per-unit change is at most $4|\alpha-\alpha'|+|\alpha'|(2\|u-u'\|+|t-t'|)$, and with $|\alpha'|\le2L$,
\begin{equation}\label{eq:unitsens}
\le 4|\alpha-\alpha'|+2L\bigl(2\|u-u'\|+|t-t'|\bigr).
\end{equation}

\emph{Grids (ball).} Discretize, per unit:
\begin{itemize}
\item $\alpha_j\in[-2L,2L]$ on a grid of mesh $\eps'/(32\bar m)$: at most $1+\dfrac{4L}{\eps'/(32\bar m)}=1+\dfrac{128L\bar m}{\eps'}$ points; contributes $\le4\cdot\dfrac{\eps'}{32\bar m}=\dfrac{\eps'}{8\bar m}$ per unit.
\item $t_j\in(-2,2)$ on a grid of mesh $\eps'/(32L\bar m)$: at most $1+\dfrac{128L\bar m}{\eps'}$ points; contributes $\le2L\cdot\dfrac{\eps'}{32L\bar m}=\dfrac{\eps'}{16\bar m}$ per unit.
\item $u_j$ on a $\dfrac{\eps'}{64L\bar m}$-net of $\Sd$: at most $\bigl(1+\dfrac{128L\bar m}{\eps'}\bigr)^{d}$ points (the standard volumetric bound $N(\Sd,\rho)\le(1+2/\rho)^d$; \cite[Cor.~4.2.13]{Vershynin}); contributes $\le2L\cdot2\cdot\dfrac{\eps'}{64L\bar m}=\dfrac{\eps'}{16\bar m}$ per unit.
\end{itemize}
Summing the three per-unit contributions gives $\le\dfrac{\eps'}{8\bar m}+\dfrac{\eps'}{16\bar m}+\dfrac{\eps'}{16\bar m}=\dfrac{\eps'}{4\bar m}$; over $\le\bar m$ units, $\le\eps'/4$.
Discretize the affine part:
\begin{itemize}
\item $v$ on a $(\eps'/16)$-net of $\{\|v\|\le L(1+2\bar m)\}\subset\R^d$: at most $\bigl(1+\dfrac{32L(1+2\bar m)}{\eps'}\bigr)^d$ points; contributes $\le2\cdot\dfrac{\eps'}{16}=\dfrac{\eps'}8$ (using $|\inn{v-v'}x|\le2\|v-v'\|$ on $B_2$).
\item $c$ on a grid of mesh $\eps'/8$ in $[-(B_0+4L\bar m),B_0+4L\bar m]$: contributes $\le\eps'/8$.
\end{itemize}
Total change $\le\eps'/4+\eps'/8+\eps'/8=\eps'/2\le\eps'$. Finally sum over the choice $m_0\in\{0,\dots,\bar m\}$ of active-unit count (a factor $\bar m+1$). Taking logarithms of the product of cardinalities,
\[
\log|\NN^\ast|\le(\bar md+2\bar m+d+1)\log\!\Bigl(\frac{C\bar m d(2+L)}{\eps'}\Bigr)+\log(\bar m+1)\le C\bar md\log\!\Bigl(\frac{C\bar md(2+L)}{\eps'}\Bigr).
\]

\emph{Grids (sphere).} On $\Sd$ we have $|\inn u x-t|\le2$, so the bulk of \eqref{eq:unitsens} is unchanged; the only issue is that a unit with $t$ near $1$ (a small spherical cap) has a large allowed $|\alpha|\le2L/\sqrt{1-t^2}$. Split the units. The plan: bulk units reuse the ball grids; cap units are first deleted when their sup-norm is negligible, and the survivors are binned dyadically in $\gamma=\sqrt{1-t}$, with the ball meshes rescaled by $\gamma_r$ inside each bin --- the rescaling exactly compensates the allowed coefficient $2L/\gamma_r$, so each bin contributes the ball-case error at the ball-case cardinality. Two bookkeeping points, once and for all. First, a grid center need not itself satisfy the $t$-dependent coefficient constraint $|\alpha|\le2L/\sqrt{1-t^2}$ at its gridded $t$: the grids produce an \emph{external} cover of $\GG_{\bar m,L}$, which suffices, since the closing paragraph of the proof converts any external $(\eps'/2)$-cover into an internal $\eps'$-net. Second, the assignment of each unit to its regime (bulk, one of the $R+1$ cap bins, or deleted) is part of the enumeration: it multiplies the count by at most $(R+3)^{\bar m}$, an additive $\bar m\log(R+3)\le C\bar m\log\bigl(C\bar md(2+L)/\eps'\bigr)$ in the logarithm, absorbed into the displayed bound.

\emph{Bulk units} ($t_j\le\tfrac12$): here $|\alpha_j|\le2L/\sqrt{1-\tfrac14}\le3L$, and the ball grids above apply verbatim (with the constant $3$ in place of $2$, absorbed into $C$).

\emph{Cap units} ($t_j\in(\tfrac12,1)$): write $\gamma:=\sqrt{1-t}\in(0,2^{-1/2})$. On $\Sd$,
\[
\sup_{x\in\Sd}\relu(\inn u x-t)=1-t=\gamma^2,\qquad |\alpha|\le\frac{2L}{\sqrt{1-t^2}}=\frac{2L}{\sqrt{(1-t)(1+t)}}\le\frac{2L}{\gamma},
\]
so the unit's sup-norm is at most $2L\gamma$. Delete every cap unit with $\gamma\le\gamma_{\min}:=\eps'/(64L\bar m)$; the total deletion cost is at most $\bar m\cdot2L\gamma_{\min}=\eps'/32$.

For the surviving cap units, use dyadic bins in $\gamma$.  Let $\gamma_r=2^r\gamma_{\min}$ and take the bins
\[
        I_r=[\gamma_r,2\gamma_r]\cap[\gamma_{\min},2^{-1/2}],\qquad r=0,1,\dots,R,
\]
where $R\le C\log(2+L\bar m/\eps')$.  In one such bin, $\gamma\in I_r$ implies $|\alpha|\le2L/\gamma_r$, $\sup\relu\le(2\gamma_r)^2=4\gamma_r^2$, and the $t$-interval has length at most $(2\gamma_r)^2-\gamma_r^2=3\gamma_r^2$.  Discretize inside the bin as follows:
\begin{itemize}
\item $\alpha\in[-2L/\gamma_r,2L/\gamma_r]$ on a grid of mesh $\eps'/(128\bar m\gamma_r^2)$.  Since $\sup\relu\le4\gamma_r^2$, this contributes at most $\eps'/(32\bar m)$; the number of grid points is at most $1+C L\bar m\gamma_r/\eps'\le C L\bar m/\eps'$.
\item $t$ on a grid of mesh $\eps'\gamma_r/(64L\bar m)$.  The ReLU map is $1$-Lipschitz in $t$, so the contribution is at most $(2L/\gamma_r)\cdot\eps'\gamma_r/(64L\bar m)=\eps'/(32\bar m)$; the number of grid points is at most $1+C L\bar m\gamma_r/\eps'\le C L\bar m/\eps'$.
\item $u$ on an $\eps'\gamma_r/(128L\bar m)$-net of $\Sd$.  The contribution is at most $(2L/\gamma_r)\cdot\eps'\gamma_r/(128L\bar m)=\eps'/(64\bar m)$; the number of net points is at most $\bigl(C L\bar m/(\eps'\gamma_r)\bigr)^d\le\bigl(C(L\bar m)^2/\eps'^2\bigr)^d$.
\end{itemize}
The factor $R$ for the choice of dyadic bin costs only $\log R\le C\log\bigl(C\bar md(2+L)/\eps'\bigr)$ after increasing constants (if $\gamma_{\min}\ge2^{-1/2}$ there are no surviving cap units).  Thus a cap unit has the same logarithmic count as in the ball case, up to the harmless factor $2d\log(C L\bar m/\eps')$ coming from the $u$-net.  Each surviving cap unit contributes at most $\eps'/(32\bar m)+\eps'/(32\bar m)+\eps'/(64\bar m)<\eps'/(8\bar m)$, and deleted caps contribute $\eps'/32$ in total.

The affine part uses the sphere bounds of Lemma~\ref{lem:affine}(ii): $\|v\|\le d(B_0+2L\bar m)$ enlarges the $v$-net's range by a factor $d$, costing an additive $\log d$. Each unit is gridded in exactly one regime (bulk, cap-bin, or deleted); enumerating these choices over at most $\bar m$ units is absorbed into $C\bar md\log(\cdots)$. Collecting terms, the sphere bound is again $\log|\NN^\ast|\le C\bar md\log\bigl(C\bar md(2+L)/\eps'\bigr)$. The sphere contributions to the $L^\infty$-error total as in the ball case---at most $\eps'/4$ over the units, $\eps'/32$ for the deleted caps, and $\eps'/8$ each for $v$ and $c$, hence at most $17\eps'/32<\eps'$.

\emph{Internal net.} First take an $(\eps'/2)$-cover of $\GG_{\bar m,L}$ with the cardinality just obtained. We now build an internal separated set inside $\AAA_{\bar m,L}$. Start with $S=\emptyset$ and, as long as there is a point of $\AAA_{\bar m,L}$ whose distance from all points already chosen is greater than $\eps'$, add such a point to $S$. Every time a point is added, the set $S$ is $\eps'$-separated. Since \eqref{eq:AinG} puts $S$ inside $\GG_{\bar m,L}$, no two points of $S$ can lie in the same $(\eps'/2)$-ball of the fixed cover; hence the process stops after at most $N(\GG_{\bar m,L},\eps'/2)$ additions. At stopping time, maximality says that every point of $\AAA_{\bar m,L}$ is within $\eps'$ of some point of $S$. Thus $S\subset\AAA_{\bar m,L}$ is an internal $\eps'$-net and
\[
        |S|\le N(\GG_{\bar m,L},\|\cdot\|_\infty,\eps'/2),
\]
which has the same logarithmic bound after adjusting the absolute constant.

\end{proof}

\subsection{Proofs for Section~\ref{sec:prob}}

\phantomsection\begin{proof}[Proof of Lemma~\textup{\ref{lem:conc-models}}]\label{proof:conc-models}\stmtbutton{lem:conc-models}
(a) L\'evy's concentration on the sphere, in sub-Gaussian form (see, e.g., \cite[Ch.~5]{Vershynin}), states that for the uniform measure on the sphere of radius $\sqrt d$ a $1$-Lipschitz function is $\psi_2$-close to its mean with an absolute constant. Given $f$ on $\Sd$ that is $L'$-Lipschitz, apply this to $z\mapsto f(z/\sqrt d)$ on the radius-$\sqrt d$ sphere, which is $(L'/\sqrt d)$-Lipschitz; this yields $\|f(x)-\EE f\|_{\psi_2}\le\kappa L'/\sqrt d$. (b) For $g\sim N(0,I_d)$ put $h(g):=f(g/\sqrt d)$, which is $(L'/\sqrt d)$-Lipschitz; the Gaussian concentration inequality (see, e.g., \cite[Ch.~5]{Vershynin}) gives $\|h-\EE h\|_{\psi_2}\le C L'/\sqrt d$, i.e.\ the claim for $x=g/\sqrt d\sim N(0,I_d/d)$.
\end{proof}

\phantomsection\begin{proof}[Proof of Lemma~\textup{\ref{lem:noise}}]\label{proof:noise}\stmtbutton{lem:noise}
Let $E_1:=\{\tfrac1n\sum_i z_i^2\ge\sigma^2-\tfrac\eps6\}$ and $E_2:=\{\tfrac1n\sum_i z_ig(x_i)\ge-\tfrac\eps6\}$. The $z_i^2\in[0,4]$ are i.i.d.\ with mean $\sigma^2$; Hoeffding's inequality \cite[Thm.~2.2.6]{Vershynin} for variables in an interval of length $4$ gives $\PP(E_1^c)\le\exp(-2n(\eps/6)^2/4^2)=e^{-n\eps^2/288}$. The $z_ig(x_i)\in[-2,2]$ are i.i.d.\ with mean $\EE[g(x)\EE[z\mid x]]=0$; likewise $\PP(E_2^c)\le e^{-n\eps^2/288}$.

On $E_1\cap E_2$, suppose $\tfrac1n\sum_i(f(x_i)-y_i)^2\le\sigma^2-\eps$ for some $f\in\FF$. Writing $y_i=g(x_i)+z_i$, so $f-y=(f-g)-z$ and $(f-y)^2=(f-g)^2-2z(f-g)+z^2$; averaging,
\[
\sigma^2-\eps\ \ge\ \underbrace{\tfrac1n\textstyle\sum_i(f-g)^2(x_i)}_{\ge0}\ -\ \tfrac2n\textstyle\sum_i z_i(f-g)(x_i)\ +\ \tfrac1n\textstyle\sum_i z_i^2 .
\]
Now $-\tfrac2n\sum z_i(f-g)=-\tfrac2n\sum z_i f+\tfrac2n\sum z_i g\ge-\tfrac2n\sum z_i f-\tfrac\eps3$ on $E_2$, and $\tfrac1n\sum z_i^2\ge\sigma^2-\tfrac\eps6$ on $E_1$. Hence
\[
\sigma^2-\eps\ \ge\ 0-\tfrac2n\textstyle\sum_i z_if(x_i)-\tfrac\eps3+\sigma^2-\tfrac\eps6,
\]
i.e.\ $\tfrac2n\sum_i z_if(x_i)\ge\eps-\tfrac\eps3-\tfrac\eps6=\tfrac\eps2$, so $\tfrac1n\sum_i z_if(x_i)\ge\tfrac\eps4$. Thus the fitting event on $E_1\cap E_2$ implies the second event; the claim follows by a union bound with $\PP(E_1^c)+\PP(E_2^c)\le2e^{-n\eps^2/288}$.
\end{proof}

\phantomsection\begin{proof}[Proof of Lemma~\textup{\ref{lem:onef}}]\label{proof:onef}\stmtbutton{lem:onef}
$\EE W_i=\EE[(f(x_i)-\EE f)\EE[z_i\mid x_i]]=0$. Two tail bounds on $W_i$: first, $|W_i|\le|z_i|\,|f(x_i)-\EE f|\le2\cdot2=4$, so $\|W_i\|_{\psi_2}\le C$ (any bounded variable is sub-Gaussian). Second, $|W_i|\le2|f(x_i)-\EE f|$, so $\PP(|W_i|\ge t)\le\PP(|f(x_i)-\EE f|\ge t/2)$; by Definition~\ref{def:conc}, $\|f(x_i)-\EE f\|_{\psi_2}\le\kappa L/\sqrt d$, hence $\|W_i\|_{\psi_2}\le 2\kappa L/\sqrt d$ up to an absolute factor. Combining, $\|W_i\|_{\psi_2}\le C\min\bigl(1,\kappa L/\sqrt d\bigr)=:K$. Standard concentration for sums of independent centered sub-Gaussian variables (see, e.g., \cite[Ch.~2]{Vershynin}) gives $\PP(\sum_i W_i\ge s)\le\exp(-c\,s^2/(nK^2))$; take $s=n\eps/8$ and note $1/K^2\ge c'\max(1,d/(\kappa^2L^2))$.
\end{proof}

\phantomsection\begin{proof}[Proof of Lemma~\textup{\ref{lem:mean}}]\label{proof:mean}\stmtbutton{lem:mean}
$|\EE f|\le1$, so $(\EE f)\tfrac1n\sum z_i\ge\tfrac\eps8$ implies $|\tfrac1n\sum z_i|\ge\tfrac\eps8$, uniformly in $f$. The $z_i\in[-2,2]$ are i.i.d.\ mean $0$; Hoeffding gives $\PP(|\tfrac1n\sum z_i|\ge\tfrac\eps8)\le2\exp(-2n(\eps/8)^2/4^2)=2e^{-n\eps^2/512}$.
\end{proof}

\phantomsection\begin{proof}[Proof of Theorem~\textup{\ref{thm:finite}}]\label{proof:finite}\stmtbutton{thm:finite}
By Lemma~\ref{lem:noise} it suffices to bound $\PP(\exists f:\tfrac1n\sum z_if(x_i)\ge\tfrac\eps4)$. Split $z_if(x_i)=W_i+(\EE f)z_i$ and $\tfrac\eps4=\tfrac\eps8+\tfrac\eps8$: if $\tfrac1n\sum z_if\ge\tfrac\eps4$ then $\tfrac1n\sum W_i\ge\tfrac\eps8$ or $(\EE f)\tfrac1n\sum z_i\ge\tfrac\eps8$. Union-bounding Lemma~\ref{lem:onef} over the $|\FF|$ functions, adding Lemma~\ref{lem:mean}, and adding the $2e^{-n\eps^2/288}$ of Lemma~\ref{lem:noise} (absorbed, with the $2e^{-n\eps^2/512}$, into $4e^{-n\eps^2/512}$) gives the bound.
\end{proof}

\subsection{Proofs for Section~\ref{sec:sharplog}}

\phantomsection\begin{proof}[Proof of Lemma~\textup{\ref{lem:hsplit}}]\label{proof:hsplit}\stmtbutton{lem:hsplit}
$|c_F|\le\sup|F|\le1$. Since $\clip$ contracts and chords bound arcs from below, $F$ is $L$-Lipschitz for geodesic distance, hence lies in $H^1(\Sd)$ with $|\nabla_TF|\le L$ almost everywhere, and Parseval for the gradient gives $\sum_{\ell\ge1}\lambda_\ell\|F_\ell\|_2^2=\EE|\nabla_TF|^2\le L^2$. The degree-one component is $F_1(x)=\inn{A_F}{x}$ with $A_F=d\,\EE[Fx]$ (because $\EE[xx^\top]=I_d/d$), and $\|F_1\|_2^2=\|A_F\|^2/d$; since $\lambda_1=d-1$, this gives $\|A_F\|^2\le dL^2/(d-1)\le\tfrac32L^2$ for $d\ge3$. For $h_F=\sum_{\ell\ge2}F_\ell$: $\EE h_F^2\le\lambda_2^{-1}\sum_{\ell\ge2}\lambda_\ell\|F_\ell\|_2^2\le L^2/(2d)$ because $\lambda_2=2d$. Finally $|h_F|\le|F|+|c_F|+\|A_F\|\le B_1$ pointwise.
\end{proof}

\phantomsection\begin{proof}[Proof of Lemma~\textup{\ref{lem:raddud}}]\label{proof:raddud}\stmtbutton{lem:raddud}
The map $F\mapsto h_F$ satisfies $\|h_F-h_G\|_\infty\le(2+\sqrt d)\|F-G\|_\infty$: indeed $|c_F-c_G|\le\|F-G\|_\infty$ and $\|A_F-A_G\|=d\,\|\EE[(F-G)x]\|=d\sup_{\|w\|=1}\EE[(F-G)\inn wx]\le d\|F-G\|_2\,\bigl(\EE\inn wx^2\bigr)^{1/2}=\sqrt d\,\|F-G\|_2$. Since $\clip$ contracts values, Proposition~\ref{prop:entropy} transfers to $\mathcal H$ with the factor $(2+\sqrt d)$ absorbed into the logarithm, giving the entropy bound. For the second claim: conditionally on $x_1,\dots,x_n$, the Rademacher process $h\mapsto\tfrac1{\sqrt n}\sum_i\varepsilon_ih(x_i)$ has sub-Gaussian increments in $L^2(P_n)$, the class is pinned at $0\in\mathcal H$ with $L^2(P_n)$-diameter at most $2\hat\sigma$, and $N(\mathcal H,L^2(P_n),u)\le N(\mathcal H,L^\infty,u)$; Dudley's entropy integral \cite[Thm.~8.1.3]{Vershynin} gives the bound with
$\int_0^{r}\sqrt{\log(B'/u)}\,du\le r\bigl[\sqrt{\log(B'/r)}+\tfrac{\sqrt\pi}2\bigr]\le2r\sqrt{\log(eB'/r)}$
(substitute $u=rv$), applied at $r=\hat\sigma$.
\end{proof}

\phantomsection\begin{proof}[Proof of Lemma~\textup{\ref{lem:selfbound}}]\label{proof:selfbound}\stmtbutton{lem:selfbound}
$\EE\sup_hP_nh^2\le\sup_h\EE h^2+\EE\sup_h(P_n-\EE)h^2$, and the first term is at most $L^2/(2d)$ by Lemma~\ref{lem:hsplit}. By symmetrization, $\EE\sup_h(P_n-\EE)h^2\le2\,\EE_{x,\varepsilon}\sup_h\tfrac1n|\sum_i\varepsilon_ih^2(x_i)|$. The map $s\mapsto s^2/(2B_1)$ is a contraction on $[-B_1,B_1]$ vanishing at $0$, so the Ledoux--Talagrand contraction principle \cite[Thm.~4.12]{LT} gives $\EE_\varepsilon\sup_h|\sum\varepsilon_ih^2(x_i)|\le4B_1\,\EE_\varepsilon\sup_h|\sum\varepsilon_ih(x_i)|$, and the claim follows.
\end{proof}

\phantomsection\begin{proof}[Proof of Theorem~\textup{\ref{thm:sharpcomplexity}}]\label{proof:sharpcomplexity}\stmtbutton{thm:sharpcomplexity}
Split $F=c_F+\inn{A_F}{x}+h_F$ by Lemma~\ref{lem:hsplit}. The affine sector obeys
\[
\EE\sup_{|c|\le1,\ \|A\|\le\sqrt{3/2}L}\frac1n\Bigl|\sum_iy_i(c+\inn A{x_i})\Bigr|\le\frac1n\Bigl(\EE\bigl|\sum_iy_i\bigr|+\sqrt{\tfrac32}L\,\EE\bigl\|\sum_iy_ix_i\bigr\|\Bigr)\le\frac{1+\sqrt{3/2}\,L}{\sqrt n},
\]
using $\EE\|\sum_iy_ix_i\|^2=\sum_i\EE\|x_i\|^2=n$. For the $\mathcal H$-sector, by symmetry of the $y_i$ it suffices to bound $\widetilde R$. Set $r^2:=L^2/(2d)$ and $a:=C\sqrt{\bar md/n}$. Write $\psi(s):=\varphi(\sqrt s)=\sqrt{\tfrac s2\log((2eB)^2/s)}$; on $(0,B^2]$ the function $s\mapsto\tfrac s2\log((2eB)^2/s)$ is increasing (its derivative is $\tfrac12[\log((2eB)^2/s)-1]>0$ for $s<(2eB)^2/e$) and concave (second derivative $-1/(2s)$), so $\psi$ is increasing and concave. By Lemma~\ref{lem:raddud}, Jensen, and Lemma~\ref{lem:selfbound},
\[
\widetilde R\ \le\ a\,\EE\varphi(\hat\sigma)\ =\ a\,\EE\psi(\hat\sigma^2)\ \le\ a\,\psi\bigl(\EE\hat\sigma^2\bigr)\ \le\ a\,\psi\bigl(r^2+8B_1\widetilde R\bigr).
\]
The key evaluation: with $r=L/\sqrt{2d}$,
\[
\log\frac{(2eB)^2}{2r^2}\;=\;2\log\frac{2eB\sqrt d}{L}\;=\;2\log\Bigl(2eC_1\bar md^{5/2}\,\frac{2+L}{L}\Bigr)\;\le\;C''\Lambda_L,
\]
since $(2+L)/L\le2(2+1/L)$ and $d^{5/2}\le d^3$. If $8B_1\widetilde R\le r^2$, then $\widetilde R\le a\psi(2r^2)\le a\,r\sqrt{C''\Lambda_L}=C'L\sqrt{\bar m\Lambda_L/n}$. Otherwise $s:=r^2+8B_1\widetilde R\le16B_1\widetilde R$ while $s\ge2r^2$, so $\log((2eB)^2/s)\le\log((2eB)^2/(2r^2))\le C''\Lambda_L$ (the logarithm decreases in $s$), whence $\psi(s)^2\le8B_1\widetilde R\,C''\Lambda_L$ and $\widetilde R\le a\sqrt{8C''B_1\widetilde R\Lambda_L}$, i.e.\ $\widetilde R\le8C''a^2B_1\Lambda_L\le C(1+L)\bar md\,\Lambda_L/n$. Collecting the three contributions proves the theorem.
\end{proof}

\phantomsection\begin{proof}[Proof of Theorem~\textup{\ref{thm:sharplog}}]\label{proof:sharplog}\stmtbutton{thm:sharplog}
Set $L^\ast:=c_0\eps\sqrt{n/(\bar m\bar\Lambda)}$ and suppose some fitting $f$ has $L:=\Lip_{\Sd}(f)\le L^\ast$. By Lemma~\ref{lem:B0} and Sections~\ref{sec:canonical}--\ref{sec:rigid}, $f|_{\Sd}\in\AAA_{\bar m,L^\ast}$, and $\clip\circ f$ fits at least as well. The noise decomposition (Lemma~\ref{lem:noise}, whose proof is pointwise in $f$) gives, outside an event of probability $2e^{-n\eps^2/288}\le\delta/4$, that $\tfrac1n\sum_iz_i\,\clip(f(x_i))\ge\eps/4$ with $z_i=y_i-g(x_i)$. Conditionally on the $x_i$ the $z_i$ are independent, mean zero, and bounded by $2$, so symmetrization and coordinate-wise contraction give
\[
\EE\ \sup_{f\in\AAA_{\bar m,L^\ast}}\frac1n\Bigl|\sum_iz_i\,\clip(f(x_i))\Bigr|\ \le\ 8\,\EE\ \sup_{f}\frac1n\Bigl|\sum_i\varepsilon_i\,\clip(f(x_i))\Bigr|\ \le\ 8\,\Xi,
\]
$\Xi$ denoting the right side of Theorem~\ref{thm:sharpcomplexity} at $L=L^\ast$. First, $\Lambda_{L^\ast}\le C\bar\Lambda$: if $L^\ast\ge1$ then $2+1/L^\ast\le3$ and $\Lambda_{L^\ast}\le\log(3e\bar md^3)\le C\bar\Lambda$; if $L^\ast<1$ then $1/L^\ast=\sqrt{\bar m\bar\Lambda}\,/(c_0\eps\sqrt n)\le\sqrt{\bar m\bar\Lambda}/c_0$ using $\eps\sqrt n\ge1$ from the first sample-size term, so $\Lambda_{L^\ast}\le\log(e\bar md^3(2+\sqrt{\bar m\bar\Lambda}/c_0))\le C\bar\Lambda$ (as $\log\bar\Lambda\le\bar\Lambda$). The three hypotheses now make the three terms of $\Xi$ each at most $\eps/(192)$: $(1+L^\ast)/\sqrt n\le\eps/192$ from the first; $L^\ast\sqrt{\bar m\Lambda_{L^\ast}/n}\le C c_0\eps\le\eps/192$ for $c_0$ small, by the definition of $L^\ast$; and $(1+L^\ast)\bar md\Lambda_{L^\ast}/n\le\eps/192$ from the second and third (for the $L^\ast$-part, $L^\ast\bar md\bar\Lambda/n=c_0\eps d\sqrt{\bar m\bar\Lambda/n}\le\eps/384$ exactly when $n\ge C\bar md^2\bar\Lambda$). Hence $\EE\sup\le\eps/8$. The supremum has bounded differences $4/n$ in each pair $(x_i,y_i)$, so McDiarmid's bounded-differences inequality \cite{McDiarmid} gives $\PP(\sup\ge\eps/4)\le e^{-n\eps^2/512}\le\delta/8$. Together with the noise event this contradicts fitting, with total failure probability at most $\delta$.
\end{proof}

\subsection{Proofs for Section~\ref{sec:genact}}

\phantomsection\begin{proof}[Proof of Theorem~\textup{\ref{thm:genact}}]\label{proof:genact}\stmtbutton{thm:genact}
\emph{Factorization and fiber transport.} $f$ depends on $x$ only through the orthogonal projection $Px$ onto $W:=\mathrm{span}\{w_1,\dots,w_m,v\}$, $\dim W\le m+1$. Suppose $x_i,x_j$ satisfy $\|Px_i\|,\|Px_j\|\le\tfrac12$. Write $x_i=z_i+w_i$ with $z_i=Px_i$, $\|w_i\|=\sqrt{1-\|z_i\|^2}\ge\tfrac{\sqrt3}2$, and set $x':=z_j+\sqrt{1-\|z_j\|^2}\,w_i/\|w_i\|\in\Sd$. Then $Px'=z_j$, so $f(x')=f(x_j)$, and
\[
\|x_i-x'\|\le\|z_i-z_j\|+\Bigl|\sqrt{1-\|z_i\|^2}-\sqrt{1-\|z_j\|^2}\Bigr|\le\Bigl(1+\tfrac1{\sqrt3}\Bigr)\|z_i-z_j\|\le2\|z_i-z_j\|,
\]
using $|\,\|z_i\|^2-\|z_j\|^2|\le(\|z_i\|+\|z_j\|)\|z_i-z_j\|$ and the lower bound on the two square roots. Hence
\begin{equation}\label{eq:fibertransport}
|f(x_i)-f(x_j)|=|f(x_i)-f(x')|\le2L\,\|P(x_i-x_j)\|,\qquad L:=\Lip_{\Sd}(f).
\end{equation}

\emph{Few high points, uniformly.} Let $\mathrm{Op}:=\lambda_{\max}(\sum_ix_ix_i^\top)$; on an event of probability $1-e^{-n/C}$, $\mathrm{Op}\le C(1+n/d)$. For any $(m{+}1)$-dimensional projection $P$ with orthonormal basis $e_1,\dots,e_{m+1}$, $\sum_i\|Px_i\|^2=\sum_k\sum_i\inn{e_k}{x_i}^2\le(m+1)\,\mathrm{Op}$, so at most $4(m+1)\mathrm{Op}\le n/8$ points have $\|Px_i\|>\tfrac12$ (the last inequality by the hypotheses on $m$). Call the others \emph{low}; there are at least $\tfrac78n$ of them, for every $P$ simultaneously.

\emph{Volumetric pairing, uniformly over a net.} We may assume $\delta_{\mathrm{cell}}:=C\sqrt{m+1}\,(8/n)^{1/(m+1)}\le\tfrac14$: otherwise the claimed bound reads $L\ge c'$ with $c'$ absolute, which already follows from one opposite-label pair (probability $1-2^{1-n}$) and $|f(x_i)-f(x_j)|=2$ with $\|x_i-x_j\|\le2$. Fix a net $\mathcal P$ of the $(m{+}1)$-frames of column-wise mesh $\delta_{\mathrm{net}}:=\delta_{\mathrm{cell}}/(8\sqrt{m+1})$, so that every admissible $P$ has $\hat P\in\mathcal P$ with $\|P-\hat P\|_{\mathrm{op}}\le\sqrt{m+1}\,\delta_{\mathrm{net}}\cdot2\le\delta_{\mathrm{cell}}/4\le1/16$; the cardinality is $e^{C(m+1)d\log(1/\delta_{\mathrm{net}})}\le e^{C(d\log n+md)}$, since $\log(1/\delta_{\mathrm{net}})\le\tfrac1{m+1}\log n+C$ (the $\sqrt{m+1}$ factors inside $\delta_{\mathrm{cell}}$ and the mesh cancel). For a fixed $\hat P\in\mathcal P$: partition the ball of radius $\tfrac12$ in $\hat P(\R^d)$ into $K=\lceil n/8\rceil$ grid cells of diameter $\delta_{\mathrm{cell}}$. Among the $\ge\tfrac78n$ low points of $\hat P$, at least $\tfrac78n-K\ge\tfrac34n-1$ share a cell with another low point, yielding at least $\tfrac38n-1\ge\tfrac n4$ disjoint same-cell pairs (for $n\ge8$), each with $\|\hat P(x_i-x_j)\|\le\delta_{\mathrm{cell}}$. These pairs are functions of $(x,\hat P)$ only; since the labels are independent of the data, the probability that fewer than $\tfrac n{16}$ of them are opposite-label is at most $e^{-n/C}$ (binomial concentration). A union bound over $\mathcal P$ costs $e^{C(d\log n+md)}$, which the sample-size hypothesis covers. Finally, for the true $P$: projected distances of unit-norm differences move by at most $2\|P-\hat P\|_{\mathrm{op}}\le\delta_{\mathrm{cell}}/2$, so the pair satisfies $\|P(x_i-x_j)\|\le\tfrac32\delta_{\mathrm{cell}}$; and each point of the pair, low for $\hat P$, has $\|Px_i\|\le\tfrac12+\tfrac1{16}\le\tfrac9{16}$, for which the transport estimate \eqref{eq:fibertransport} holds with the constant $2$ unchanged (the square roots in its proof are bounded below by $\sqrt{1-(9/16)^2}\ge\tfrac45$, giving factor $1+\tfrac{9/16\cdot2}{2\cdot4/5}\le1.71\le2$). So with probability $1-2e^{-n/C}$, for \emph{every} admissible $P$ there is an opposite-label pair with $\|P(x_i-x_j)\|\le\tfrac32\delta_{\mathrm{cell}}$ and both points $\tfrac9{16}$-low.

\emph{Conclusion.} For that pair, exact fitting gives $|f(x_i)-f(x_j)|=|y_i-y_j|=2$, while \eqref{eq:fibertransport} gives $2\le2L\cdot\tfrac32\delta_{\mathrm{cell}}$, i.e.\ $L\ge\tfrac23\delta_{\mathrm{cell}}^{-1}=c\,n^{1/(m+1)}/\sqrt{m+1}$.
\end{proof}

\phantomsection\begin{proof}[Proof of Theorem~\textup{\ref{thm:genact2}}]\label{proof:genact2}\stmtbutton{thm:genact2}
Write $L=\Lip_{\Sd}(f)$. As in the proof of Theorem~\ref{thm:genact}, $f$ factors through the orthogonal projection $P$ onto a subspace $W\supseteq\mathrm{span}\{w_1,\dots,w_m,v\}$, enlarged to $\dim W=p$. In each part $\delta$ denotes a scale fixed there; we may assume $\delta\le\tfrac15$, since otherwise the stated bound is at most an absolute constant, which (with the advertised $c$ taken small enough that the target is $\le1$ in this range) follows from one opposite-label pair ($L\ge1$).

\emph{Localization.} Let $\mathrm{Op}:=\lambda_{\max}(\sum_ix_ix_i^\top)$; as before $\mathrm{Op}\le C_1(1+n/d)$ with probability $1-2e^{-d}$. For every rank-$p$ orthogonal projector $Q$, $\sum_i\|Qx_i\|^2=\mathrm{tr}(Q\sum_ix_ix_i^\top Q)\le p\,\mathrm{Op}$; so, deterministically on the event $\mathrm{Op}\le C_1(1+n/d)$, the count $\#\{i:\|Qx_i\|>\tau\}\le p\,\mathrm{Op}/\tau^2\le n/8$ with $\tau:=(8C_1p(\tfrac1n+\tfrac1d))^{1/2}$, uniformly over all rank-$p$ projectors $Q$ at once. The hypotheses give $r_0:=2\sqrt{p/d}\le\tau\le\tfrac18$.

\emph{Net and grid.} Fix a column-wise $\delta/(8\sqrt p)$-net $\mathcal P$ of the $p$-frames as in Theorem~\ref{thm:genact}, so every admissible $P$ has $\hat P\in\mathcal P$ with $\|P-\hat P\|_{\mathrm{op}}\le\delta/4$; in both parts $\delta\ge(n^2\sqrt d)^{-1}$, so $|\mathcal P|\le e^{Cpd\Lambda}$. For $\hat P\in\mathcal P$ fix orthonormal coordinates on its range, let $z_i$ be the coordinates of $\hat Px_i$, and partition $\R^p$ into half-open cubical cells of side $\delta/\sqrt p$; $n_c$ is the number of $z_i$ in cell $c$ and $q_c:=\PP(z\in c)$ for an independent copy. From each cell with $n_c\ge2$ take disjoint same-cell pairs greedily; the family $\Pi(x,\hat P)$ is determined by $(x,\hat P)$.

\emph{Part \eqref{eq:floorI}.} Set $\delta:=4\sqrt p\,\tau(8/n)^{1/p}$ and cell side $s:=\delta/\sqrt p=4\tau(8/n)^{1/p}$, so that $\tau/s=\tfrac14(n/8)^{1/p}\ge1$ (here $p\le c\log n$ with $c$ small enough that $(n/8)^{1/p}\ge4$). A radius-$\tau$ ball meets at most $(2\tau/s+2)^p\le(4\tau/s)^p=n/8$ of the side-$s$ cells, the first inequality using $\tau/s\ge1$. At least $\tfrac78n$ points are $\tau$-low, so among them at least $\tfrac78n-\tfrac n8$ share cells, giving $|\Pi|\ge\tfrac12(\tfrac78n-\tfrac n8)=\tfrac38n$.

\emph{Part \eqref{eq:floorII}.} Set $N^*:=\lceil C_0pd\Lambda\rceil$ and $\delta:=14\sqrt p\,r_0\,(N^*/(c_2n^2))^{1/p}$ with $c_2:=(128e^2)^{-1}$; the hypotheses give $\delta\le\sqrt p\,r_0$. Since $\EE\|z\|^2=p/d$ and $r_0^2=4p/d$, Markov gives core mass $\sum_{c\cap B_{r_0}\ne\emptyset}q_c\ge\tfrac34$. Call a cell \emph{light} if $q_c\le1/n$, \emph{mid} if $1/n<q_c\le32/n$, \emph{heavy} if $q_c>32/n$; one class carries core mass $\ge\tfrac14$.
\emph{Light:} core cells number at most $(7\sqrt p\,r_0/\delta)^p$; for light cells $\PP(n_c\ge2)\ge\binom n2q_c^2(1-q_c)^n\ge n^2q_c^2/(8e)$, so by Cauchy--Schwarz the expected number $\mu$ of doubly occupied light core cells satisfies $\mu\ge\frac{n^2}{128e}\bigl(\frac{\delta}{7\sqrt p\,r_0}\bigr)^p=\frac{n^2}{128e}\cdot\frac{2^pN^*}{c_2n^2}\ge2^{p+1}N^*$. Multinomial occupancy counts are negatively associated \cite{JDP}, monotone functions of disjoint coordinates preserve negative association, and the Chernoff--Hoeffding lower tail transfers \cite{DR}; hence $\PP(\#\{\text{doubly occupied}\}\le N^*)\le e^{-\mu/8}$.
\emph{Mid:} at least $(\tfrac14)/(32/n)=n/128$ mid core cells; for each, $nq_c\ge1$ gives $\PP(n_c\ge2)\ge1-(1-\tfrac1n)^n-(1-\tfrac1n)^{n-1}\ge\tfrac15$ for $n\ge100$, so $\mu\ge n/640\ge2N^*$ and the same negative-association tail applies.
\emph{Heavy:} heavy cells number at most $n/32$; the number of points in heavy core cells dominates $\mathrm{Bin}(n,\tfrac14)$, hence is $\ge n/8$ with probability $1-e^{-n/32}$, and $\ge n/8$ points in $\le n/32$ cells yield at least $\tfrac12(\tfrac n8-\tfrac n{32})=\tfrac{3n}{64}\ge N^*$ disjoint pairs.
In every case $|\Pi|\ge N^*$ with probability $1-2e^{-cN^*}$, all pairs in $B_{r_0+\delta}$.

\emph{Labels, union, conclusion.} Given $x$, the pairs are disjoint and $y$ is independent, so $\PP(\text{no opposite-label pair}\mid x)\le2^{-|\Pi|}$; the union over $\mathcal P$ costs $e^{Cpd\Lambda}$, absorbed by $N^*$ (take $C_0$ large) in part \eqref{eq:floorII} and by $\tfrac38n\ge N^*$ in part \eqref{eq:floorI}. On the good event, for the true $P$ some opposite-label pair has $\|P(x_i-x_j)\|\le\tfrac32\delta$ with both points $(\tau+\tfrac54\delta)$-low, and \eqref{eq:fibertransport} gives $2\le2L\cdot\tfrac32\delta$, i.e.\ $L\ge\tfrac2{3\delta}$. Substituting the two choices of $\delta$ (and $\tau\le4\sqrt{C_1p/d}$, valid as $n\ge d$) gives \eqref{eq:floorI} and \eqref{eq:floorII}.
\end{proof}

\subsection{Proofs for Section~\ref{sec:scope}}

\phantomsection\begin{proof}[Proof of Proposition~\textup{\ref{prop:upper}}]\label{proof:upper}\stmtbutton{prop:upper}
The profile $\rho$ vanishes on $(-\infty,1-s]$, rises with slope $1/s=4$ on $[1-s,1]$, and equals $1$ at $u=1$; on $\Sd$ no argument exceeds $1$.

\emph{Interpolation.} Since $\rho(1)=1$ and $\rho(\inn{x_j}{x_i})=0$ for $j\ne i$ by $\inn{x_j}{x_i}\le1/8<3/4$, we have $f(x_i)=y_i$.

\emph{Disjoint caps.} Let $C_i=\{x\in\Sd:\inn{x_i}{x}>3/4\}$. If $x\in C_i\cap C_j$ with $i\ne j$, then $\inn{x_i+x_j}{x}>3/2$, so $\|x_i+x_j\|>3/2$. But
\[
        \|x_i+x_j\|^2=2+2\inn{x_i}{x_j}\le 2+\frac14=\frac94,
\]
contradiction. Thus the caps are pairwise disjoint.

\emph{Lipschitz bound.} Along any unit-speed geodesic $\gamma$ on $\Sd$, the derivative of $y_i\rho(\inn{x_i}{\gamma(t)})$ exists for a.e. $t$ and, on the rising band $\inn{x_i}{\gamma(t)}\in[3/4,1]$, has absolute value at most
\[
        4|y_i|\sqrt{1-\inn{x_i}{\gamma(t)}^2}\le4\sqrt{1-(3/4)^2}=\sqrt7.
\]
Off the rising band it is $0$ a.e. Because the caps are disjoint, at every point of the sphere at most one summand is nonconstant. Hence $|\frac{d}{dt}f(\gamma(t))|\le\sqrt7$ for a.e. $t$, and $f$ is $\sqrt7$-Lipschitz for geodesic distance. For chord distance, if the geodesic distance between $x,x'$ is $\theta\in[0,\pi]$, then $\|x-x'\|=2\sin(\theta/2)$ and $\theta/(2\sin(\theta/2))\le\pi/2$. Therefore $|f(x)-f(x')|\le(\pi/2)\sqrt7\|x-x'\|$.

\emph{Separation probability.} For two independent uniform points, conditioning on $x_j$ and applying spherical concentration to the $1$-Lipschitz function $x\mapsto\inn{x}{x_j}$ gives an absolute tail bound $\PP(|\inn{x_i}{x_j}|>t)\le2e^{-c d t^2}$ (see, e.g., \cite[Ch.~5]{Vershynin}). At $t=1/8$ this is at most $2e^{-c d/64}$. A union bound over at most $n^2/2$ pairs gives failure probability at most $n^2e^{-c d/64}$, which is at most $1/n$ as soon as $d\ge C_{\rm sep}\log n$ for a sufficiently large absolute constant $C_{\rm sep}$.
\end{proof}

\section*{Acknowledgments and funding}
The research presented in this paper was supported by the European Research Council (ERC) under the European Union's Horizon 2022 research and innovation programme (grant agreement No.~101041711), by the Simons Foundation as part of the Collaboration on the Mathematical and Scientific Foundations of Deep Learning, by Heights Labs, by the Israel Science Foundation (grant number 2258/19), by the Israel Science Foundation (ISF Grant 4101/25), and by the U.S. National Science Foundation (NSF Grant OISE-2401227).

\section*{Declaration of competing interest}
The author declares no competing interests.

\section*{Declaration of generative AI and AI-assisted technologies in the manuscript preparation process}
During the preparation of this work the author used generative AI tools to accelerate drafting and revision. All mathematical claims, proofs, numerical interpretations, and bibliographic information were subsequently reviewed and edited by the author, who takes full responsibility for the content of the manuscript.

\section*{Data and code availability}
The numerical scripts used for the checks reported in Section~\ref{sec:numerics} are available at \url{https://github.com/yspennstate/law-of-robustness-two-layer} and are also included with the source package accompanying this manuscript.

\end{document}